\documentclass{article} 
\usepackage{nips15submit_e,times}
\usepackage{url}
\usepackage{amssymb}
\usepackage{amsthm}
\usepackage{amsfonts}
\usepackage{complexity}
\usepackage{graphicx}
\usepackage{dcolumn}
\usepackage{bm}
\usepackage{hyperref}
\usepackage{enumerate}
\usepackage{algorithm}
\usepackage{algpseudocode}

\newtheorem{theorem}{Theorem}
\newtheorem{lemma}{Lemma}

\newtheorem{corollary}{Corollary}

\author{
Nathan Wiebe\\
Microsoft Research\\
Redmond, WA 98052 \\
\texttt{nawiebe@microsoft.com} \\
\And
Ashish Kapoor \\
Microsoft Research\\
Redmond, WA 98052\\
\texttt{akapoor@microsoft.com} \\
\And
Christopher E. Granade \\
University of Sydney \\
Sydney, NSW 2006\\
\texttt{cgranade@cgranade.com} \\
\And
Krysta M. Svore \\
Microsoft Research\\
Redmond, WA 98052\\
\texttt{ksvore@microsoft.com}
}

\begin{document}

\def\ket#1{\left|#1\right\rangle}
\def\bra#1{\langle#1|}
\newcommand{\ketbra}[2]{|#1\rangle\!\langle#2|}
\newcommand{\braket}[2]{\langle#1|#2\rangle}
\newcommand{\prob}[1]{{\rm Pr}\left(#1 \right)}
\newcommand{\expect}[2]{{\mathbb{E}_{#2}}\!\left\{#1 \right\}}
\newcommand{\var}[2]{{\mathbb{V}_{#2}}\!\left\{#1 \right\}}
\newcommand{\CRej}{{\rm RejS}}
\newcommand{\CSMC}{{\rm SMC}}

\newcommand{\sinc}{{sinc}}


\newcommand{\sde}{\mathrm{sde}}
\newcommand{\Z}{\mathbb{Z}}
\newcommand{\RR}{\mathbb{R}}
\newcommand{\w}{\omega}
\newcommand{\Kap}{\kappa}

\newcommand{\Tchar}{$T$}
\newcommand{\T}{\Tchar~}
\newcommand{\TT}{\mathrm{T}}
\newcommand{\ClT}{\{{\rm Clifford}, \Tchar\}~}
\newcommand{\Tcount}{\Tchar--count~}
\newcommand{\Tcountper}{\Tchar--count}
\newcommand{\Tcounts}{\Tchar--counts~}
\newcommand{\Tdepth}{\Tchar--depth~}
\newcommand{\Zr}{\Z[i,1/\sqrt{2}]}
\newcommand{\ve}{\varepsilon}

\newcommand{\eq}[1]{\hyperref[eq:#1]{(\ref*{eq:#1})}}
\renewcommand{\sec}[1]{\hyperref[sec:#1]{Section~\ref*{sec:#1}}}
\newcommand{\app}[1]{\hyperref[app:#1]{Appendix~\ref*{app:#1}}}
\newcommand{\fig}[1]{\hyperref[fig:#1]{Figure~\ref*{fig:#1}}}
\newcommand{\thm}[1]{\hyperref[thm:#1]{Theorem~\ref*{thm:#1}}}
\newcommand{\lem}[1]{\hyperref[lem:#1]{Lemma~\ref*{lem:#1}}}
\newcommand{\tab}[1]{\hyperref[tab:#1]{Table~\ref*{tab:#1}}}
\newcommand{\cor}[1]{\hyperref[cor:#1]{Corollary~\ref*{cor:#1}}}
\newcommand{\alg}[1]{\hyperref[alg:#1]{Algorithm~\ref*{alg:#1}}}
\newcommand{\defn}[1]{\hyperref[def:#1]{Definition~\ref*{def:#1}}}

\newcommand{\targfix}{\qw {\xy {<0em,0em> \ar @{ - } +<.39em,0em>
\ar @{ - } -<.39em,0em> \ar @{ - } +
<0em,.39em> \ar @{ - }
-<0em,.39em>},<0em,0em>*{\rule{.01em}{.01em}}*+<.8em>\frm{o}
\endxy}}

\newenvironment{proofof}[1]{\begin{trivlist}\item[]{\flushleft\it
Proof of~#1.}}
{\qed\end{trivlist}}

\newcommand{\cu}[1]{{\textcolor{red}{#1}}}
\newcommand{\tout}[1]{{}}
\newcommand{\good}{{\rm good}}
\newcommand{\bad}{{\rm bad}}
\newcommand{\dd}{\mathrm{d}}

\newcommand{\id}{\openone}
\title{Quantum Inspired Training for Boltzmann Machines}

\nipsfinalcopy 
\maketitle
\begin{abstract}
We present an efficient classical algorithm for training deep Boltzmann machines (DBMs) that uses rejection sampling in concert with variational approximations to estimate the gradients of the training objective function.
Our algorithm is inspired by a recent {\it quantum} algorithm for training DBMs \cite{wiebe_quantum_2014}.
We obtain rigorous bounds on the errors in the approximate gradients; in turn, we find that choosing the instrumental distribution to minimize the $\alpha=2$ divergence
with the Gibbs state minimizes the asymptotic algorithmic complexity.
Our rejection sampling approach can yield more accurate gradients than low-order contrastive divergence training and  
the costs incurred in finding increasingly accurate gradients can be easily parallelized.
Finally our algorithm can train full Boltzmann machines and scales more favorably with the number of layers
in a DBM than greedy contrastive divergence training.  

\end{abstract}
\section{Introduction}
\label{sec:intro}

In 2002, Hinton provided the first efficient algorithm for training Boltzmann machines~\cite{hinton_training_2002}, a type of stochastic recurrent neural network with undirected edges, called contrastive divergence (CD).  
Due to the emergence of CD, Boltzmann machines, specifically restricted Boltzmann machines (RBMs) and their deep layered counterparts (DBMs), have become standard tools
for solving problems in vision and speech recognition~\cite{JH11,EHW+14}.  Despite the remarkable successes that contrastive divergence has achieved, there are a number of theoretical
and practical drawbacks to the use of the algorithm.
A major drawback of contrastive divergence training is that it implicitly adds directionality to the edges in the graphical model, which lessens any advantages Boltzmann machines may have over feed--forward neural nets.  On a more practical level, parallelism cannot be used to increase the accuracy of contrastive divergence training which means that only
the lowest--order (and least accurate) contrastive divergence approximation is used.

Recent work in quantum computing has revealed a class of training algorithms 
that use a quantum form of rejection sampling to overcome these problems~\cite{wiebe_quantum_2014}.
The approach hinges on refining quantum states that crudely approximate the joint and conditional probability distributions for the units in the Boltzmann machine into ones that closely mimic them.  
This process is expected to yield accurate and efficient approximations to the  distribution if sufficiently strong regularization is used in the training process~\cite{WH02,Jor99}.
The gradients of the training objective function (the average log--likelihood of the BM producing the training data) are then estimated by either sampling from the resulting quantum states or
by using techniques like quantum amplitude amplification and estimation.
Amplitude amplification results in a quadratic speedup with respect to the acceptance probability of the rejection step and amplitude estimation quadratically reduces the number of training vectors (at the price of quadratically worsening the scaling with $\mathcal{E}$)~\cite{wiebe_quantum_2014}.  These advantages are summarized in~\tab{runtimes}.

This quantum approach has several significant features.  
First, it is a very natural method for training a Boltzmann machine using a quantum computer because it only involves state preparation and measurement.
Second, it does not explicitly depend on the interaction graph used:  it can be used to train
full Boltzmann machines rather than just DBMs.  
Third, it is mathematically easy to verify that the approximate gradients converge to the true gradients in the appropriate limit.
Finally, it does not have a well known classical counterpart, unlike many existing quantum machine learning results.

Unfortunately, quantum computers are currently limited to tens of quantum bits, which means that~\cite{wiebe_quantum_2014} can
 only be used to train impractically small Boltzmann machines using present-day hardware.
Consequently, determining a classical analogue of the quantum approach would embue classical training methods with the quantum advantages.
The challenge is whether the classical analogue would still result in efficient training.
We present such a classical analogue in this paper.

We call our approach \emph{Instrumental Rejection Sampling} (IRS). It retains all of the theoretical advantages of the quantum algorithm, while providing practical advantages for training highly optimized
deep Boltzmann machines in the presence of regularization.   It is worth noting that our approach is not specific to Boltzmann machines: it also applies to more general classes of undirected graphical models with latent variables.

Here we investigate the quality of the gradients yielded by this method and provide theoretical and numerical evidence that training using IRS 
is not only efficient, but it may also convey practical advantages for training 
certain classes of deep Boltzmann machines.  
Since these insights stemmed from recent progress on quantum machine learning, our work underscores the value of investigating quantum paradigms for machine learning even in the absence of large--scale quantum computers.

\begin{table}[t!]
\centering
\begin{tabular}{|c|c|c|c|}
\hline
&Greedy CD-$k$ training & IRS training&Quantum training\\
\hline
Complexity &$O(N_{\rm epoch} N_{\rm train} k \mathcal{E} \ell)$ & $O(N_{\rm epoch} N_{\rm train} \kappa\mathcal{E} )$&$O(N_{\rm epoch}N_{\rm train} \mathcal{E} \sqrt{\kappa} )$\\
\hline
Depth &$O(N_{\rm epoch}k \ell \log[ N_{\rm train} \mathcal{E}])$ &$O(N_{\rm epoch} \log[N_{\rm train} \kappa\mathcal{E}] )$&$O(N_{\rm epoch} \log[N_{\rm train} \kappa\mathcal{E}] )$\\
\hline
\end{tabular}
\caption{Time complexities for training an $\ell$--layer DBM with $\mathcal{E}$ edges using greedy contrastive divergence, our algorithm and quantum training.  The latter two algorithms are not restricted to DBMs.  Depth is the time needed for a cluster with an unbounded number of computational nodes.\label{tab:runtimes}}
\end{table}
\section{Instrumental Rejection Sampling}\label{sec:theory}
The main insight behind our result, and stemming from~\cite{wiebe_quantum_2014}, is that variational approximations to the Gibbs state can be used to make 
training with rejection sampling much more efficient than it would initially appear.  This result is conceptually related to work by Murray and Gharhramani~\cite{MG04} in the context
of MCMC algorithms; whereas, our results hold for rejection sampling and also show how to choose the approximation to minimize the error in the sample distribution.
We present the general theory of IRS here and apply it to training DBMs in~\sec{IRS}.

Rejection sampling seeks to draw samples from a distribution $P(x)/Z:= P(x)/\sum_x P(x)$ that cannot be sampled from directly by
sampling instead from an instrumental distribution, $Q$, and rejecting the samples with a probability $P(x)/Z\kappa Q(x)$.  Here $\kappa$ is
a normalizing constant introduced to ensure that the rejection probability is well defined.  In other words, we draw samples from $Q(x)$ and reject them with probability
\begin{equation}
{\rm Pr}_{\rm accept}(x|Q(x),\kappa, Z) = \frac{P(x)}{Z \kappa Q(x)},\label{eq:rejprob}
\end{equation}
until a sample is accepted.  This can be implemented by drawing $y$ uniformly from $[0,1]$ and accepting $x$ if $y\le {\rm Pr}_{\rm accept}(x|Q(x),\kappa, Z_Q)$.

In applications such as training Boltzmann machines, the constants needed to normalize~\eq{rejprob} are not known or may be prohibitively large for some $x$.  
This can be addressed using a form of approximate rejection sampling where we use $\kappa_A < \kappa$ and $Z_Q \approx Z$ such that $ \frac{P(x)}{Z_Q \kappa_A Q(x)}>1$ for some set of configurations which we call \emph{bad}.  The approximate rejection sampling algorithm then proceeds as the precise rejection sampling algorithm except that the sample $x$ will always be accepted if $x\in {\rm bad}$.  This means that the samples yielded by approximate rejection sampling are not precisely drawn from $P/Z$. Error estimates for this sampling process are given below.

\begin{theorem}
Let $Q(x):x\in \mathbb{Z}^+_{2^n}$ be an efficiently computable probability distribution that can be efficiently sampled from.  Assume that $P(x):x\in \mathbb{Z}^+_{2^n}$ can be efficiently computed and $P(x)/Z_Q Q(x) > \kappa_A~\forall~x\in {\rm bad} \subseteq\mathbb{Z}^+_{2^n}$ where 
$$\sum_{x\in {\rm bad}}  \left(P(x)-\kappa_A Z_Q Q(x)\right) \le \epsilon Z.$$
  There then exists an efficient classical algorithm that samples from a distribution $\tilde{P}$ such that ${\sum_x \sqrt{\tilde P(x) P(x) }} \ge {\sqrt Z}(1-\epsilon)$. The probability of accepting a sample is at least $\frac{Z(1-\epsilon)}{Z_Q\kappa_A}$.\label{thm:kappa}
\end{theorem}
Our proof of this theorem uses quantum techniques and is given in the appendix.
\thm{kappa} shows that if the conditions normally required of rejection sampling are not met then an approximate sampling algorithm exists that is promised to be close to the distribution if the chosen value of $\kappa_A$ is sufficiently large.  Since the acceptance rate of the sampling algorithm scales inversely with $\kappa_A$, elementary choices of $Q$ (such as the uniform distribution) are unlikely to efficiently produce accurate samples from $P(x)/Z$ for polynomially large $\kappa_A$.  We call our approach {Instrumental Rejection Sampling} training to emphasize that we use a non-trivial $Q$ to draw the samples.

In order to minimize the error in~\thm{kappa}, we want to choose $Q$ to minimize an appropriate divergence between $Q$ and $P/Z$.
The choice of divergence that $Q$ minimizes is by no means unique.  The result of~\cite{wiebe_quantum_2014} chooses $Q$ to be a product distribution that minimizes ${\rm KL}(Q||P/Z)$ known as the mean--field distribution; however, this choice only provides a polynomial penalty for using an exponentially poor approximation to the tail probability.  

We can address this problem by using the method of~\cite{Min05} to find $Q$ to minimize a divergence that does not de--emphasize these tail probabilities.  This method finds $Q$ that minimizes $D_{\alpha}(P/Z||Q)$ where 
\begin{equation}
D_{\alpha}(p\,||\, q) =\frac{\int_x \alpha p(x) + (1-\alpha)q(x) -p(x)^\alpha q(x)^{1-\alpha} \mathrm{d}x}{\alpha(1-\alpha)}.
\end{equation}
Note that $\lim_{\alpha\rightarrow 0} D_{\alpha}(p||q) = {\rm KL}(q||p)$ and hence the method of~\cite{Min05} generalizes the mean--field approximation.
The following lemma shows that choosing $Q=Q_{\alpha=2}$ to be the $Q$ that minimizes $D_2(P/Z||Q)$ also minimizes an upper bound on the rejection sampling error.
\begin{lemma}\label{lem:alpha}
Let $Q(x): x\in \mathbb{Z}^+_{2^n}$ and $P(x)/Z : x\in \mathbb{Z}^+_{2^n}$ be probability distributions then
$$
\frac{\sum_{x\in {\rm bad}} \left(P(x) -\kappa_A Z_Q Q(x)\right)}{Z}\in O\left(\frac{Z  D_2(P/Z\, ||\, Q)}{\kappa_AZ_Q}\right),
$$
where $D_2(P/Z \, || \, Q)=\frac{1}{2}\sum_x (P(x)/Z - Q(x))^2/Q(x)$ is the $\alpha=2$ divergence and we consider the asymptotic regime where $D_2(P/Z \, || \, Q)\gg 1$.
\end{lemma}
\begin{proof}
There are two cases, ${\rm bad} =\emptyset$ and ${\rm bad}\ne \emptyset$.  If we have the former case then $\sum_{x\in {\rm bad}} \left(P(x) -\kappa_A Z_Q Q(x)\right)=0$ and the result trivially follows.  Now let us focus on the case where the set is non--empty.
Since $Q(x)$ is a probability distribution and $\kappa_A\ge 1$
\begin{equation}
\frac{\sum_{x\in {\rm bad}} P(x) -\kappa_AZ_Q Q(x)}{Z}\le \frac{\sum_{x\in {\rm bad}} P(x)}{Z}.\label{eq:Pbound0}
\end{equation}
In order to estimate the probability, it is helpful to think of the problem as a sampling problem for the random variable $P(x)/Q(x)$ where $x\sim P(x)/Z$.  The desired probability is then bounded above (using Markov's inequality) by
\begin{eqnarray}
\sum_{x\in {\rm bad}} P(x)/Z &=& \Pr_{x\sim P(x)/Z}(P(x)/Z_QQ(x) > \kappa_A)\nonumber\\
&\le& \Pr_{x\sim P(x)/Z}(P(x)/Z_QQ(x) \ge \kappa_A)\nonumber\\
&\le& \frac{\mathbb{E}( P/Q)}{\kappa_AZ_Q}=\sum_x\frac{ P^2(x)}{ZQ(x)\kappa_AZ_Q}.\label{eq:Pbound}
\end{eqnarray}
Rewriting the $\alpha$--divergence as a sum and using the fact that $P/Z$ and $Q$ are normalized to $1$ leads to
\begin{equation}
D_2(P/Z \, || \, Q) = \frac{1}{2}\sum_x \frac{P^2(x)}{Z^2 Q(x)}-\frac{1}{2}.\label{eq:D2alpha}
\end{equation}
Thus~\eq{Pbound} and~\eq{D2alpha} imply
\begin{equation}
\frac{\sum_{x\in {\rm bad}} P(x)}{Z} \le\frac{Z(2D_2(P/Z \, || \, Q)+1)}{\kappa_AZ_Q},
\end{equation}
which completes the proof under the assumption that $D_2(P/Z||Q) \gg 1$.
\end{proof}
~\lem{alpha} shows that choosing $Q$ to minimize $D_2$ will asymptotically minimize the approximation error.  In contrast, no corresponding result has been rigorously shown for ${\rm KL}(Q||P/Z)$.  

\section{Approximate training of Boltzmann machines using IRS \label{sec:IRS}}
We now show how instrumental rejection sampling can be applied to training Boltzmann machines.
Boltzmann machines model the training data as an Ising model in thermal equilibrium with its environment.  
The goal of training is to adjust the parameters of the Ising model to maximize the likelihood that the observed training data would emerge from
the thermal distribution in the system.  

A Boltzmann machine consists of spins (bits) which are composed of $n_v$ visible units and $n_h$ hidden units~\cite{hinton_training_2002,Ben09}.  
The visible units represent the input and output of the model and the hidden units are used to generate appropriate correlations
between the features in the input and output.  The correlations can be visualized as edges in a graphical model between pairs of hidden and visible units.  In general, the underlying graph can be 
a complete graph; however, in practice a layered bipartite graph is usually preferred since it admits efficient training.

The restricted Boltzmann machine (RBM) consists of two layers of units, where each layer consists of either exclusively hidden or exclusively visible units.
Edges are restricted to inter-layer correlations.
The unnormalized probability of a given configuration of hidden and visible units is
\begin{equation}
P(v,h)={e^{-E(v,h)}},\label{eq:gibbs}
\end{equation}
where $P/Z$ for $Z= \sum_{v,h} e^{-E(v,h)}$ is the normalized joint probability distribution and
\begin{equation}
E(v,h) = -\sum_i b_i v_i -\sum_j d_j h_j -\sum_{i,j} w_{i,j} v_i h_j,\label{eq:Energy}
\end{equation}
for binary units $v_i \in \mathbb{Z}_2$ and $h_j \in \mathbb{Z}_2$.  The vectors $b$ and $d$ are called biases and set the probability
of a unit being zero or one irrespective of the values of the adjacent units.  The weights $w$ provide energy penalties for two adjacent units taking the value $1$.

Training the model formally involves maximizing the average log--likelihood of the training data emerging from the model.
This can be thought of as an optimization problem with objective function
\begin{equation}
O_{\rm ML} = -\frac{1}{N_{\rm train}}\sum_{x\in x_{\rm train}} \log\left(\sum_h e^{-E(x,h)}\right) - \log(Z) -\frac{\lambda}{2}w^Tw,
\end{equation}
where $\lambda w^Tw/2$ is a regularization term introduced to combat overfitting.

Unfortunately, calculation of the objective function is in general intractable as calculation of the partition function $Z$ is $\#P$ complete for 
non--planar Ising models.  Nonetheless, the gradient of $O_{\rm ML}$ can be estimated without knowing the log--partition function~\cite{hinton_training_2002}.
\begin{equation}
\frac{\partial O_{\rm ML}}{\partial w_{i,j}} = \langle v_i h_j \rangle_{\rm data} -\langle v_i h_j \rangle_{\rm model} - \lambda w_{i,j}.\label{eq:grad}
\end{equation}
Here the expectation value over the data refers to the average of the corresponding pair of visible and hidden units given the constraint that the
visible units are clamped to the values of the individual training vectors (the hidden units are allowed to vary and take values given by the Gibbs distribution).
The expectation value over the model corresponds to the average value of the product of the pair of units for the Ising model when the visible units
are not constrained to take values in the training set.  The derivatives of $O_{\rm ML}$ in the directions of the biases are given in the appendix.

Despite the innocuous appearance of the gradient in~\eq{grad}, it can still be challenging to compute because the expectation values
require drawing samples from a Gibbs distribution described by~\eq{gibbs}, which is itself an \NP-hard task in general.  This intractability is sidestepped through the use of approximate methods such as contrastive divergence~\cite{hinton_training_2002,WH02}.  CD training is efficient, but it has theoretical and practical shortcomings \cite{ST10} that we aim to address with Instrumental Rejection Sampling.

%

\subsection{Approximate Gibbs distributions}
Many methods can be employed to provide an estimate of the Gibbs distribution $P/Z$.  Perhaps the simplest is the mean--field approximation (MF), which takes $Q$ to be a factorized probability distribution over all hidden and visible units in the graphical model.  Among these factorized distributions, $Q$ is taken to be the one that is closest to $\Pr(v,h)=e^{-E(v,h)}/Z$, where closest means that it minimizes ${\rm KL}(Q\, ||\, e^{-E}/Z)$.  In particular, for an RBM
$$Q_{\rm MF}(v,h) = \prod_{j=1}^{n_v} \mu_j^{v_j}(1-\mu_j)^{1-v_j}\prod_{k=1}^{n_h} \nu_k^{h_k}(1-\nu_k)^{1-h_k},$$
where
\begin{eqnarray}
\mu_j = (1+e^{-b_j - \sum_{k} w_{jk} \nu_k})^{-1},\nonumber\\
\nu_k = (1+e^{-d_k - \sum_{j} w_{jk} \mu_j})^{-1}.\label{eq:munu}
\end{eqnarray}
 The optimal mean--field parameters can be approximated by solving~\eq{munu} using fixed point iteration.
The MF approximation for generic Boltzmann machines takes a very similar form~\cite{WH02}.  Note that here $Q(v,h)$ is 
a product distribution, but multi--modal or structured mean--field approximations can be used instead in
cases where the MF approximation is expected to break down~\cite{Jor99}.

Although the mean--field approximation is expedient to compute, \lem{alpha} suggests that choosing $Q$ to minimize $D_2$ will asymptotically minimize the algorithm's complexity.   We refer to the product distribution that $D_2(P/Z||Q)$ as $Q_{\alpha=2}$.  The minimization strategy used to find $Q_{\rm MF}$ does not work for $Q_{\alpha=2}$ because $D_2$ does not contain logarithms and hence more general methods, such as fractional belief propagation, are needed.    Fractional belief propagation works by choosing $Q$ to variationally minimize an upper bound on the log--partition function that corresponds to the choice $\alpha=2$.  The algorithm is explained in detail in~\cite{WH03,Min05}.


In order to maximize the probability of success, it is useful to have an estimate of  the partition function $Z$.
The log--partition function can be estimated for any product distribution $Q$ using 
\begin{equation}
\log(Z) \ge \log(Z_{Q}):= \sum_x Q(x) \log\left(\frac{e^{-E(x)}}{Q(x)} \right)= -\langle E \rangle -H[Q(x)],\label{eq:variationalbd}
\end{equation}
where $H[Q(x)]$ is the Shannon entropy of $Q(x)$ and $\langle E \rangle$ is the expected energy in the state $Q$.   Equation~\eq{variationalbd} holds with equality if and only if $Q(x) = e^{-E(x)}/Z$.  The slack in the inequality is the Kullback--Leibler divergence ${\rm KL}(Q\, ||\, e^{-E(x)}/Z)$, which means that if $Z_Q = -\langle E \rangle -H[Q(x)]$ then the approximation will become more accurate as $Q$ approaches the Gibbs distribution~\cite{Min05}.  We denote this mean--field approximation to the partition function as $Z_{\rm MF}$.
Other tractable approximations to the log--partition function can also be used~\cite{Min05,wainwright_tree_2013}.   For simplicity, we use $Z_{\rm MF}$ for the majority of the subsequent numerical experiments.

\begin{algorithm}[t!]
\rule{\linewidth}{1pt}
\begin{algorithmic}
\vskip0.2em
\hrule
\vskip0.2em
\Repeat
\State Compute distribution $Q$ from $w$ $b$ and $d$.
\State Compute $Z_Q$ from $Q$.
\For{each Training vector $x\in x_{\rm train}$}
\State Compute distribution $\mathcal{Q}(h|x)$ from $x$, $w$, $b$ and $d$.
\State Compute $Z_{\mathcal{Q}(h|x)}$ from $\mathcal{Q}(h|x)$.
\Repeat
\State Attempt to  sample from $e^{-E(x,h)}/\sum_{h} e^{-E(x,h)}$ using~\eq{rejprob}
\State with instrumental distribution $\mathcal{Q}(h|x)$ and $Z_{\mathcal{Q}(h|x)}\kappa_A$.
\Until{sample is accepted}
\Repeat
\State Attempt to  sample from $P/Z$  using~\eq{rejprob} with instrumental distribution $Q$ and $Z_Q\kappa_A$.
\Until{sample is accepted}
\EndFor
\State Compute gradients using expectation values of accepted samples and~\eq{grad}.
\State Update weights and biases using a gradient step with learning rate $r$.
\Until{Converged or maximum epochs reached.}
\end{algorithmic}
\rule{\linewidth}{1pt}
\caption{\label{alg:RS}Rejection Sampling Training Algorithm ${\rm IRS}(w,b,d,\{x_{\rm train}\})$.}
\end{algorithm}

\subsection{Training algorithm}
Our training algorithm for Boltzmann machines is given in~\alg{RS}.  The algorithm assumes that the user has (1) a method $Q(v,h)$ that approximates the model distribution and (2) a family of distributions $\mathcal{Q}(h;v)$ that estimates the data distribution when the visible units are clamped to data vector $v$.  We assume in both cases that the resulting distributions are product distributions.  Note that $P(h|v)$ is expected to be approximately unimodal for trained models~\cite{WH02} and hence it is reasonable to expect that it will provide a good approximation to the true probability.  We assume the that drawing a sample from $Q(v,h)$ or $\mathcal{Q}(h;v)$ costs one operation.  Arithmetic operations such as addition, multiplication and exponentiation are each assumed to cost a single operation.

\begin{theorem}
The expected number of query and arithmetic operations required to train a $\ell$--layer connected DBM containing $\mathcal{E}$ edges using rejection sampling for $N_{\rm epoch}$ epochs is $O\left({N_{\rm epoch} N_{\rm train} \kappa_A\mathcal{E} }\right)$,
if the assumptions of~\thm{kappa} hold with $(1-\epsilon) \in \Omega(1)$ and the mean--field approximation is used to estimate the partition function.
\end{theorem}
The proof of the theorem is a straightforward exercise in counting the number of expected steps in~\alg{RS} and is given for completeness in the appendix.

In contrast, the cost of using greedy contrastive divergence training with CD-$k$ is $O(N_{\rm epoch} N_{\rm train} k \mathcal{E} \ell)$~\cite{Ben09}.  Thus IRS training will have an advantage over contrastive divergence for training sufficiently deep networks if $\kappa_A\in \Theta(k)$.   It is worth noting that MF-CD training~\cite{WH02} also offers advantages for training deep networks but can be less accurate than CD training~\cite{TH09}.

A further advantage of IRS over CD training is that the rejection sampling step can be parallelized.  Conversely, the $k$ sampling steps used in CD-$k$ cannot be parallelized.  Parallelism can boost the accuracy of IRS training without increasing the runtime of the algorithm, as discussed in~\tab{runtimes}.

\subsection{Accuracy of gradients}
As with contrastive divergence training, the accuracy of the gradients yielded by IRS training is controllable.  The two parameters that affect the quality of the gradients are $\kappa_A$ and $N_{\rm samp}$.  As both of these quantities approach infinity, the error in the gradient goes to zero.  We formalize this in the following corollary, which is proven in the appendix.
\begin{corollary}\label{cor:gradbd}
The expected Euclidean norm of the difference between the gradient computed by $N_{\rm samp}$ samples using rejection sampling with instrumental distribution $Q$ using $\kappa_A$ and $Z\approx Z_Q$ for a connected binary Boltzmann machine with $\mathcal{E}$ edges computed using $P/Z$ is
$$
O\left(\sqrt{\mathcal{E}} \left(\frac{1}{\sqrt{N_{\rm samp}}}+\frac{Z D_2(P/Z \, || \, Q)}{\kappa_A Z_Q} \right) \right).
$$
\end{corollary}
The proof of~\cor{gradbd} follows directly from~\lem{alpha}, standard error bounds on statistical sampling and norm inequalities.  Note that the choice $N_{\rm samp} = N_{\rm train}$ in~\alg{RS} is not necessary.

\fig{graderror} confirms the expectations of~\cor{gradbd} by showing that the error in the estimated gradient shrinks as $O(1/\sqrt{N_{\rm samp}})$ if $\kappa_A$ is sufficiently large.  Furthermore, because there are twice as many weights in the RBM with $n_v=6$ and $n_h=4$ as there are in the RBM with $n_v=4$ and $n_h=3$, the ratio of the two errors should be a factor of $\sqrt{2}$.  The numerical experiments on these small RBMs agree with this assumption, suggesting that the scaling with $\mathcal{E}$ also agrees with~\cor{gradbd}.  Also, $\kappa_A=10$ proves to be sufficient for the majority of the data sets considered, while $\kappa_A=4$ appears to be the threshold below which the quality of the gradient is negatively impacted by $\kappa_A$.

We examine the value of $D_2$ obtained using $Q_{\rm MF}$ and $Q_{\alpha=2}$ case in~\tab{kappa} for a random ensemble of RBMs in order to assess the benefits of IRS training using $Q_{\alpha=2}$.  We choose a large weight distribution for the data to emphasize the discrepancies between the two.  For more modest weight distributions, the two quantities become comparable (see appendix). We find that using the mean--field approximation instead of the optimal distribution increases the expected $\kappa$ predicted by~\cor{gradbd} by up to $7$ orders of magnitude for a $22$ unit RBM.  In particular, the values of $\kappa_A$ needed are sufficiently low such that the majority of these distributions can be exactly prepared from $Q_{\alpha=2}$ using a reasonable number of samples.  

\begin{figure}[t!]
\centering
\includegraphics[width=0.45\columnwidth]{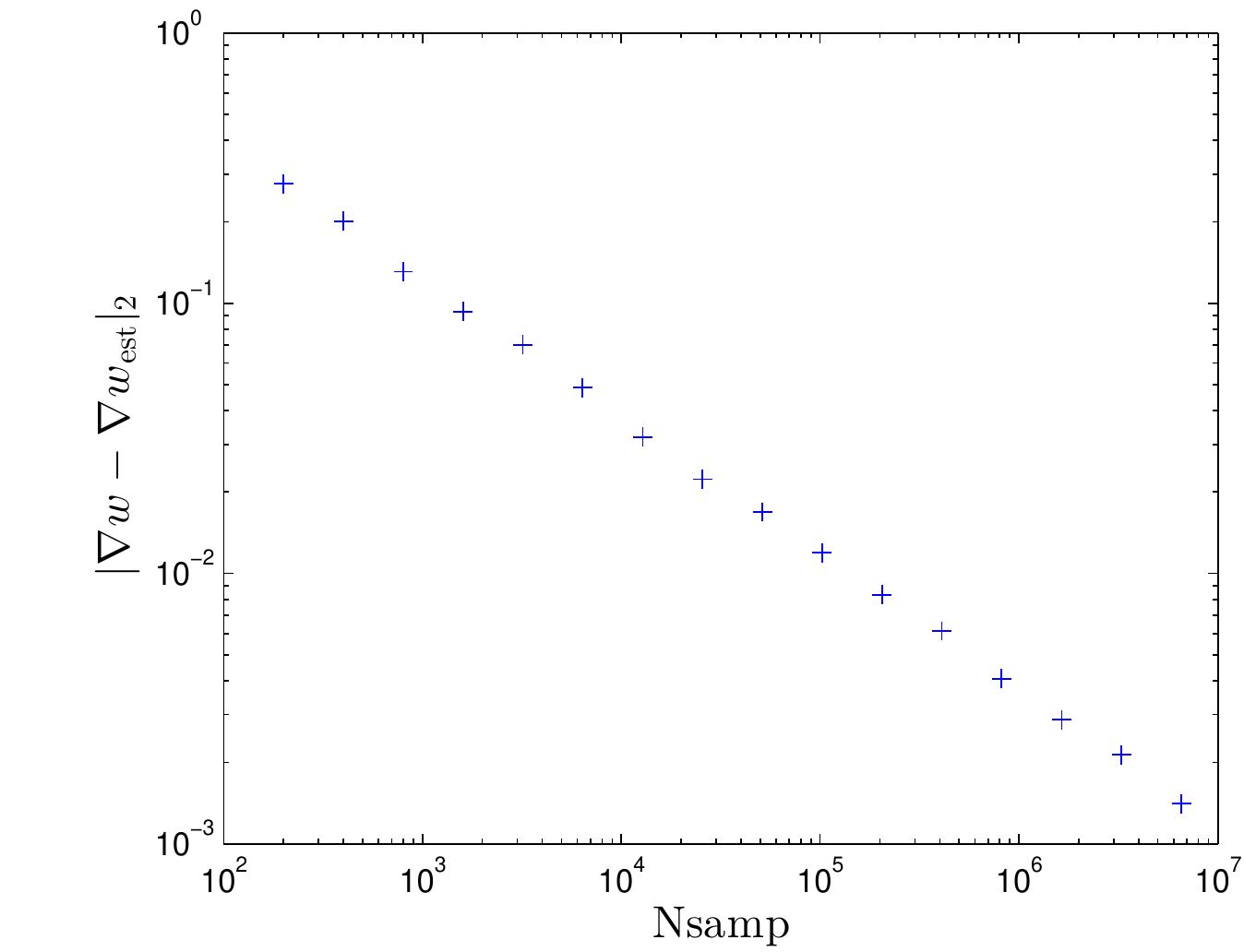}
\includegraphics[width=0.45\columnwidth]{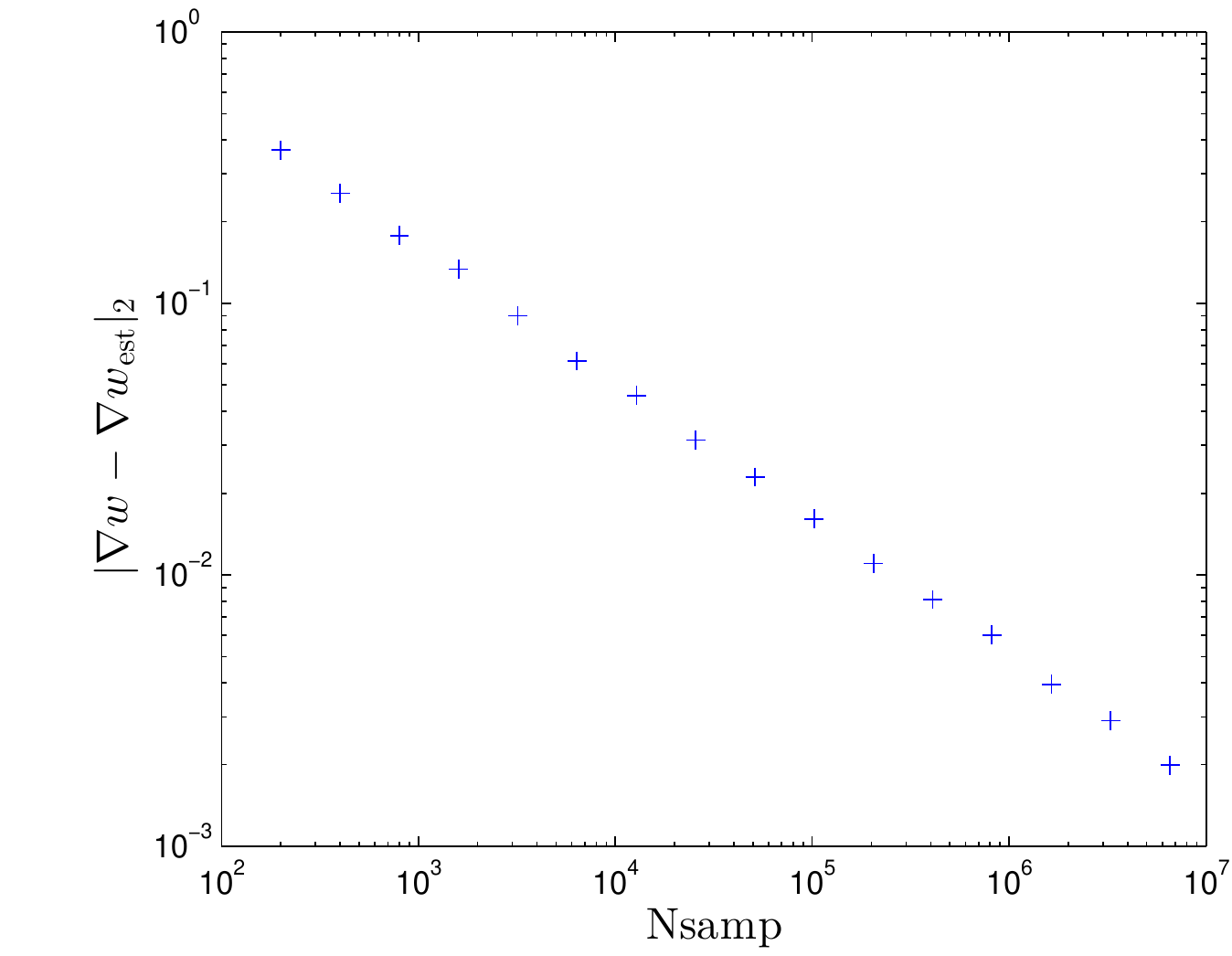}
\caption{Mean difference between the gradient of the weight vector computed by rejection sampling, $\nabla w_{\rm est}$, and the gradient computed directly from~\eq{grad} for $100$ randomly generated RBMs as a function of the number of samples considered in the rejection sampling algorithm with $\kappa=10$. Weights and biases are taken to be $\mathcal{N}(0,1)$. (Left) $n_v=4$, $n_h=3$ and (right) $n_v=6$, $n_h=4$. \label{fig:graderror}}
\end{figure}

\begin{table}[t!]
\centering
\begin{tabular}{|c||c|c|}
\hline
$n_h$ &$\log_{10}\left(D_2(P/Z\, ||\, Q_{\rm MF})\right)$&$\log_{10}\left(D_2(P/Z\, ||\, Q_{\alpha=2})\right)$\\
\hline
$4$ & $0.39\pm 0.88$ & $-0.44 \pm 0.33$\\
$8$ & $1.7\pm 1.1$& $ -0.06 \pm 0.33$\\
$12$ & $3.1\pm 1.6$ & $0.21 \pm 0.36$\\
$16$ & $4.6\pm 1.9$ &$0.34 \pm 0.41$\\
\hline
\end{tabular}
\caption{$\alpha$--divergences for synthetic RBMs with $n_v=6$.  Biases and weights are normal with zero mean and unit variance. The $\alpha=2$ data scales linearly; whereas mean--field scales exponentially.  \label{tab:kappa}}
\end{table}

\subsection{Accuracy of training}

We will now compare the performance of IRS training algorithm to CD-1 training for small restricted Boltzmann machines.   We use the following synthetic training data:
\begin{eqnarray}
{[x_1]_j} &=& 1 \mbox{ if } j \le n_v/2, \mbox{ else } 0 \nonumber\\
{[x_2]_j} & = & j \mbox{ mod } 2\label{eq:4datavectors},
\end{eqnarray}
and their bitwise complements.  We further add Bernoulli noise of strength $\mathcal{N}=0.1$ to each of the bits in the training vectors to make the data set more challenging to learn.

We assess the quality of the training methods by using~\alg{RS} or CD-1 to find an approximation to the optimal model parameters and then compute the exact gradients of $O_{\rm ML}$) to find the location of the true optima that these methods estimate.   The data is presented in~\fig{CDML} for a small RBM with $n_v=6$ and $n_h=4$.  The data shows that while contrastive divergence finds an optima that is on average within $0.13\%$ of that exact ML training can provide, IRS training deviates by $0.0015\%$ and continues to converge to the ML optima as $N_{\rm epochs}$ increases.  Further numerical evidence is provided in the appendix. Such differences are expected to be even more striking for deep restricted Boltzmann machines because IRS does not greedily optimize the weights~\cite{wiebe_quantum_2014}.

\section{Outlook}
Our results open a number of further avenues for further inquiry.
Although IRS has asymptotic advantages over existing methods for training DBMs in the presence of sufficiently strong regularization, our work only shows that it can provide accurate and efficient approximations to the gradients of the training objective under such circumstances.  Further work will be needed to examine the performance of IRS training in practical machine learning problems.

Since IRS achieves its goals by combining results from the disparate fields of variational approximations and deep learning, it is natural to suspect that it could be optimized by going
beyond the simple unimodal approximations used in the main body.  Indeed, we see in the appendix that the use of bimodal approximations can substantially reduce
the sample complexity of training.  Further study may reveal even more practical variational approximations to $P(x)/Z$.

Our work also illustrates that quantum machine learning may be an important avenue of inquiry for understanding machine learning in a broader context.
This utility of thinking from a quantum perspective is not likely to be unique to training Boltzmann machines since classical computing is in a subset of quantum computing and hence every result shown for classical machine learning also applies to a subset of quantum machine learning.  
Conversely, lower bounds and no-go theorems proven for the quantum setting also apply to the classical setting, which makes quantum computing ideally suited for understanding the 
limitations and opportunities that physics places on a machine's ability to learn.
Just as quantum insights inspired the present work, re-examining other areas of machine learning through the lens of quantum computing may not only lead to new 
classical algorithms but also provide deep insights into the nature of learning and inference.  

\begin{figure}[t!]
\centering
\includegraphics[width=0.45\columnwidth]{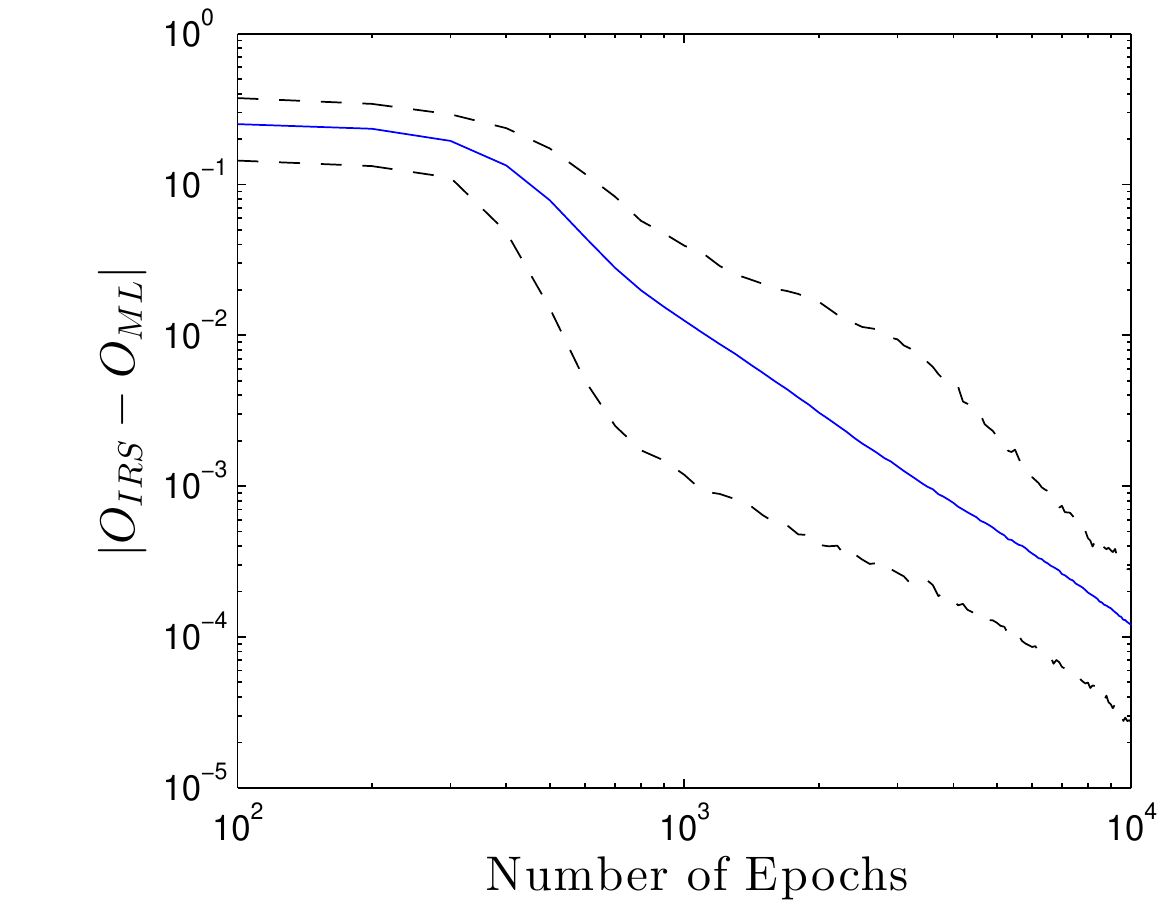}
\includegraphics[width=0.45\columnwidth]{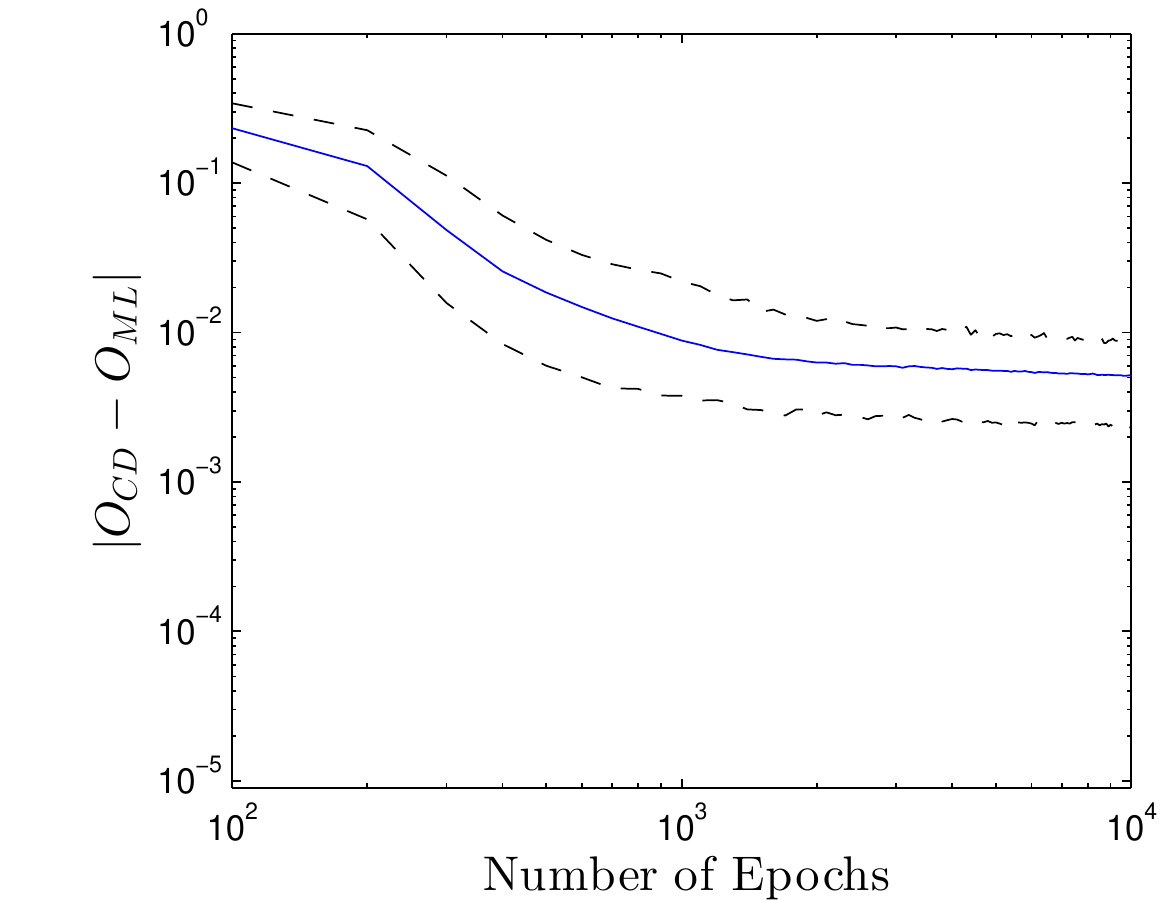}
\caption{Difference between the objective functions computed using IRS gradients (left), CD-1 gradients (right) and the objective function evaluated at the optima found by ML training.  Dashed lines denote a $95\%$ confidence interval; solid lines denote the mean.  For IRS, $\kappa_A=800$, $N_{\rm samp} = N_{\rm train}=100$, $\mathcal{N}=0.1$ and $Q$ is an equal mixture of the mean--field and uniform distributions,  $\lambda=0.05$ and the learning rate varies from $0.1$ to $0.001$ at $10000$ epochs.   $\mathbb{E}(O_{\rm ML})\approx -3.8724$.\label{fig:CDML}}
\end{figure}

\subsubsection*{Acknowledgments}
We thank Tom Minka for valuable discussions and for the code for computing $\alpha$--divergences.
\appendix

\section{Review of Dirac Notation}

We will derive much of the theory behind our method using language from quantum computing.  This language not only
proves to be useful for representing rejection sampling based algorithms, but also is useful because it provides
a clear method for dequantizing the method of~\cite{wiebe_quantum_2014}.
The central object in quantum computing is the quantum state (which we will take to mean a \emph{pure state} in the following).  
A quantum state is simply a unit vector in $\mathbb{C}^N$ such that the magnitude squared of each of its components yields
the probability of the corresponding outcome.

A quantum state is typically represented (using Dirac notation) as $\ket{\psi}$ which can be interpreted to be a column vector in $\mathbb{C}^N$.  Similarly, $\bra{\psi}$ is its Hermitian transpose.
The notation is linear, which means that if $\{\ket{j}\}$ is a complete orthonormal vector space on $\mathbb{C}^N$ then the state $\ket{\psi}$ can
be written as
\begin{equation}
\ket{\psi} = \sum_{j=1}^N \sqrt{P_j} \ket{j}.\label{eq:superposition}
\end{equation}
Physically, such states describe the probability distribution of outcomes that emerges when the state $\ket{\psi}$ is measured in this basis where each $P_j$ gives the probability of finding the system in basis state $\ket{j}$.
Classically, this measurement process is equivalent to sampling from a probability distribution but in quantum mechanics this is more subtle
because the resultant probability distribution changes depending on the basis that the system is measured in; whereas in classical applications there is implicitly only one such basis.
The fact that Dirac notation is basis independent gives it a distinct advantage for concisely describing distributions relative to column--vector notation.

Dirac notation also has an implicit tensor product structure built in, meaning that
\begin{equation}
\ket{\psi}\ket{\phi} := \ket{\psi}\otimes \ket{\phi}.\label{eq:tensor}
\end{equation}
This convention is useful for describing large sets of uncorrelated variables because the tensor product structure implicitly captures the lack of correlations between the variables
described by $\ket{\psi}$ and those described by $\ket{\phi}$.  Correlated distributions can be described by combining~\eq{superposition} and~\eq{tensor} via
\begin{equation}
\sum_{i,j} a_{i,j} \ket{u_i}\ket{v_j},
\end{equation}
where $\ket{u_i}$ and $\ket{v_j}$ represent orthonormal basis vectors that span the space that the variables described by $\ket{\psi}$ and $\ket{\phi}$ are supported in.
Also, for the probability of drawing a particular sample $v_j$ from the marginal distribution over $v_j$ is, for example, $\Pr(v_j)= \sum_i |a_{i,j}|^2$ and the resultant marginal state over the $u_i$ is
\begin{equation}
\sum_{i} \frac{a_{i,j}}{\sqrt{\sum_i |a_{i,j}|^2}} \ket{u_i}.  
\end{equation}

\section{Proofs}
\subsection{De--quantization of Quantum Rejection Sampling and Proof of Theorem 1}
While the states in quantum computing are the fundamental objects of the computational model, quantum computers also need a complete set of operations that are capable of performing an arbitrary (unitary) transformation on input states.  This means that a quantum computer has to have the ability to transform an arbitrary input unit vector into any other such vector.  For the present purposes, however, it suffices to note that probability distributions are first class entities in quantum computing and that quantum computers provide a method for preparing any such distribution.  As a consequence, it should come as no surprise that notation developed to describing quantum devices should also be useful for describing classical sampling algorithms.

Quantum rejection sampling is one of the most important tools used to design quantum algorithms not only because it can be used to enable exponential speedups for certain tasks, but also because it is very natural in quantum computing~\cite{ORR13}.  Let us assume that we have a (potentially un--normalized) distribution $P$ that we cannot directly prepare and also have an efficiently preparable distribution $Q$.  Furthermore, let us also assume that a constant is known such that $P/Q \le \kappa$.  We then prepare the state
\begin{equation}
\sum_{x} \sqrt{Q(x)}\ket{x} \left(\sqrt{\frac{P(x)}{Q(x) \kappa}}\ket{1} + \sqrt{1-\frac{P(x)}{Q(x) \kappa}}\ket{0} \right).\label{eq:qrej}
\end{equation}
If we measure the right--most register, which we will refer to as \emph{the coin}, to be one then the resultant state over the \emph{sample register} is
\begin{equation}
\sum_{x} \sqrt{\frac{P(x)}{\sum_x P(x)}}\ket{x}:=\sum_{x} \sqrt{\frac{P(x)}{Z}}\ket{x},
\end{equation}
and thus if we consider the state post--selected on successfully measuring the right most register in~\eq{qrej} to be $0$ then we prepare the desired distribution $P(x)$ over the remaining register.  The probability of success for this is $\sum_x \frac{P(x)}{\kappa}$.

Although quantum notation is used in this procedure, there is nothing inherently quantum about it as written.  In fact, the exponential advantages that quantum algorithms accrued through quantum rejection sampling arise only because the distributions $P$ or $Q$ cannot be efficiently computed or because the distribution $Q$ cannot be efficiently sampled from. This is summarized in the following lemma.
\begin{lemma}
Let $Q(x):x\in \mathbb{Z}^+_{2^n}$ be an efficiently computable probability distribution that can be efficiently sampled from.  Assume that $P(x):x\in \mathbb{Z}^+_{2^n}$ can be efficiently computed and $P(x)/Q(x) \le \kappa~\forall~x\in\mathbb{Z}^+_{2^n}$ then the task of drawing samples from the distribution yielded by quantum rejection sampling can be efficiently simulated by a classical computer.\label{lem:qrej}
\end{lemma}

\begin{proof}
Let us assume that we are provided with a state as per~\eq{qrej}.  Drawing a sample from the correct distribution corresponds to measuring the coin register in~\eq{qrej} and conditioned on measuring $1$ the sample register is measured and the result is output as the sample from $P(x)$.  Because the partial trace is a commutative operation, the order of these two measurements is arbitrary.  Instead, we could first measure the sample register resulting and if the result is $x$ then the marginal distribution over the coin is $\sqrt{\frac{P(x)}{Q(x) \kappa}}\ket{1} + \sqrt{1-\frac{P(x)}{Q(x) \kappa}}\ket{0}$.  Now the coin register can be measured and the sample $x$ will be accepted if and only if the result is $1$, which occurs with probability $P(x)/(Q(x)\kappa)$.  This process is equivalent to the original quantum algorithm.

By assumption the distribution $Q(x)$ can be sampled from efficiently be a classical computer.  This means that the first step in the re--ordered quantum algorithm can be efficiently simulated.  Next, using the fact that $P(x)/(Q(x)\kappa) \le 1$ and that $P(x)$ and $Q(x)$ are efficiently computable it follows that we can efficiently simulate drawing a sample from the coin register by sampling from a Bernoulli distribution with $p= \frac{P(x)}{Q(x) \kappa}$.  This process, also known as rejection sampling, is efficient and hence quantum rejection sampling can be efficiently simulated using a classical computer under these assumptions.
\end{proof}

This consequently shows that the GEQS algorithm for training deep networks given in~\cite{wiebe_quantum_2014} can be efficiently simulated and has a direct classical analog in the IRS algorithm.  Similarly, the GEQAE algorithm in~\cite{wiebe_quantum_2014} can also be efficiently simulated but IRS is not a natural analog of it since GEQAE uses a manifestly quantum method (known as amplitude estimation) for inferring the expectation values needed to train the DBM.  This method may be of particular importance because it can lead to quadratic reductions in the number of times the database of training vectors needs to be queries, which leads to significant cost savings for typical machine learning problems.

Despite the fact that~\lem{qrej} shows that a classical computer can often simulate quantum rejection sampling efficiently quantum computers can nonetheless provide huge advantages for rejection sampling.  In fact, speedups relative to classical algorithms can arise from the use of quantum subroutines such as amplitude amplification to quadratically reduce the rejection rate in the sampling process~\cite{Poulin}.  However, since we focus on classical algorithms for rejection sampling we will ignore such quantum algorithms in the following.

\begin{proofof}{Theorem 1}
The proof here follows one given in~\cite{wiebe_quantum_2014}.
The algorithm described in~\lem{qrej} will fail as writted if $P/Q > \kappa_A$ because the marginal distribution over the coin register will no longer be normalized.  This can be addressed, at the price of introducing errors in the distribution, by clipping the probabilities used in the coin register to $[0,1]$ for the configurations $x\in {\rm bad}$ where $P/Q > Z_Q \kappa_A$.
Using Dirac notation, we wish to sample from the following state.
\begin{equation}
\sum_{x\in {\rm good}} \sqrt{Q(x)} \ket{x} \left(\sqrt{\frac{P(x)}{ Q(x)Z_Q\kappa_A}} \ket{1}  + \sqrt{1-\frac{P(x)}{ Q(x)Z_Q\kappa_A}}\ket{0}\right)+ \sum_{x\in{\rm bad}} \sqrt{Q(x)} \ket{x} \ket{1},
\end{equation}
where ${\rm good} = \mathbb{Z}_{2^n} \setminus {\rm bad}$.
The probability of measuring the coin to be $1$ is 
\begin{eqnarray}
P_{\rm accept} &=& \sum_{x\in {\rm good}} \frac{P(x)}{Z_Q\kappa_A} + \sum_{x\in{\rm bad}} Q(x) \ge \sum_{x}\frac{P(x)(1-\epsilon)}{Z_Q\kappa_A}. 
\end{eqnarray}
We denote the resultant state
\begin{equation}
\ket{\tilde{P}} = \sum_x \sqrt{\tilde{P}(x)} \ket{x} :=\frac{\sum_{x\in {\rm good}} \sqrt{P(x)} \ket{x} +  \sum_{x\in {\rm bad}} \sqrt{Z_Q\kappa_A Q(x)} \ket{x}}{\sqrt{\sum_{x}{P(x)+\sum_{x\in{\rm bad}}(Z_Q\kappa_A Q(x) - P(x))}}}.
\end{equation}
Now the fidelity of the marginal state that results from post--selected measurement of the coin register (i.e. accepting the sample from $Q$) is
\begin{eqnarray}
\braket{\tilde P}{P} &=& \frac{\sum_{x} \sqrt{\tilde{P}(x) P(x)}}{\sqrt{\sum_x P(x)}}\ge \frac{\sum_{x\in {\rm good}} {P(x)}  +  \sum_{x\in {\rm bad}} {Z_Q\kappa_A Q(x)} }{\sum_x P(x)\sqrt{1-\epsilon}}\ge 1-\epsilon.
\end{eqnarray}

This shows that there exists a quantum algorithm that has the desired success probabilities and incurs error at most $\epsilon$ in the resultant distribution.  The remainder of the proof then follows by the same logic as that used in~\lem{qrej} to show that the algorithm can be de--quantized by exchanging the order that the coin and sample registers are measured in.  Thus there is an equivalent classical algorithm under the assumptions that $P$ and $Q$ are efficiently computable and $Q$ can be efficiently sampled from.
\end{proofof}

\subsection{Proof of Corollary 1}
\begin{proofof}{Corollary 1}
Let us focus on the problem of approximating the model expectation present in the gradient with respect to $w_{i,j}$.  The triangle inequality and~Theorem 1 imply that if $\tilde{P}$ is the probability distribution that is obtained after rejection sampling is performed using $\kappa_A$ and $\{y_k: k=1,\ldots,N_{\rm samp}\}$ are the samples drawn from $\tilde{P}(x)$ in the rejection sampling protocol
\begin{eqnarray}
&&\Biggr|\frac{1}{N_{\rm samp}}\sum_{k=1}^{N_{\rm samp}} \delta_{x,y_{(k)}}x_i x_j  - \langle x_i x_j\rangle_{\rm model}\Biggr| \nonumber\\
&&\qquad\le \Biggr|\frac{1}{N_{\rm samp}}\sum_{k=1}^{N_{\rm samp}} \delta_{x,y_{(k)}}x_i x_j-\sum x_i x_j (\tilde{P}(x)) +\sum x_i x_j \tilde{P}(x) - \langle x_i x_j\rangle_{\rm model}\Biggr| \nonumber\\
&&\qquad\le \biggr|\sum x_i x_j (\tilde{P}(x) - \delta_{x\in \{y_k\}}/N_{\rm samp})\biggr| + \biggr|\sum x_i x_j (\tilde{P}(x) - P/Z) \biggr|.\nonumber\\
&&\qquad\le \biggr|\sum x_i x_j (\tilde{P}(x) - \delta_{x\in \{y_k\}}/N_{\rm samp})\biggr| + \biggr|\sum (\tilde{P}(x) - P/Z) \biggr|.\nonumber\\
&&\qquad\in O\left( \biggr|\sum x_i x_j (\tilde{P}(x) - \delta_{x\in \{y_k\}}/N_{\rm samp})\biggr| + \frac{Z D_2(P/Z \, || \, Q)}{\kappa_AZ_Q}\right).\label{eq:1termbd}
\end{eqnarray}
Because the units are binary the variance of $x_i x_j$ is at most $1$, which means that the sampling error can be bounded and hence~\eq{1termbd} is
\begin{equation}
 O\left(\frac{1}{\sqrt{N_{\rm samp}}} + \frac{Z D_2(P/Z \, || \, Q)}{\kappa_AZ_Q} \right).\label{eq:2termbd}
\end{equation}
\end{proofof}
The exact same argument can be applied to the model average and the gradients of the biases.  The conclusion is identical in each case, that the contribution to the error from sampling and using an insufficient value of $\kappa_A$ is at most~\eq{2termbd}.  For a connected graph, the maximum number of components of any of these vectors is $\mathcal{E}$.  The triangle inequality and the fact that $\|\cdot \|_{2} \le \sqrt{\mathcal{E}} \|\cdot \|_{\max}$ then gives us our result.

\section{Additional Numerics}

\begin{figure}[t!]
\includegraphics[width=0.495\columnwidth]{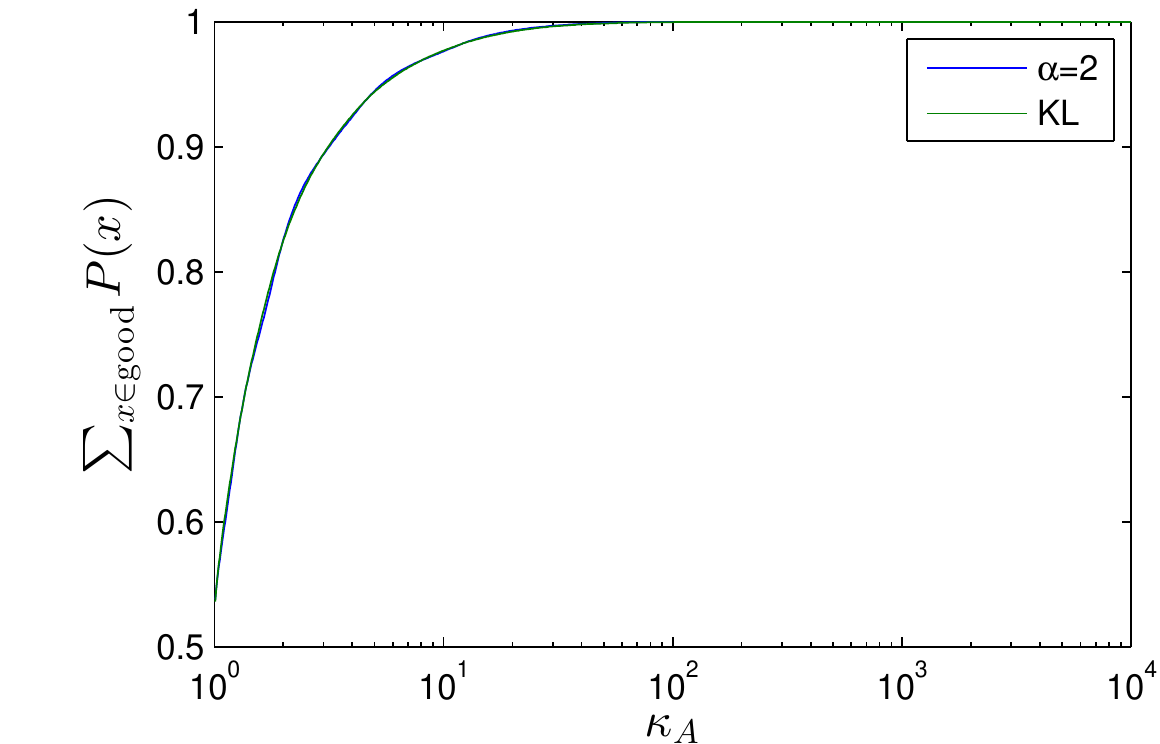}
\includegraphics[width=0.495\columnwidth]{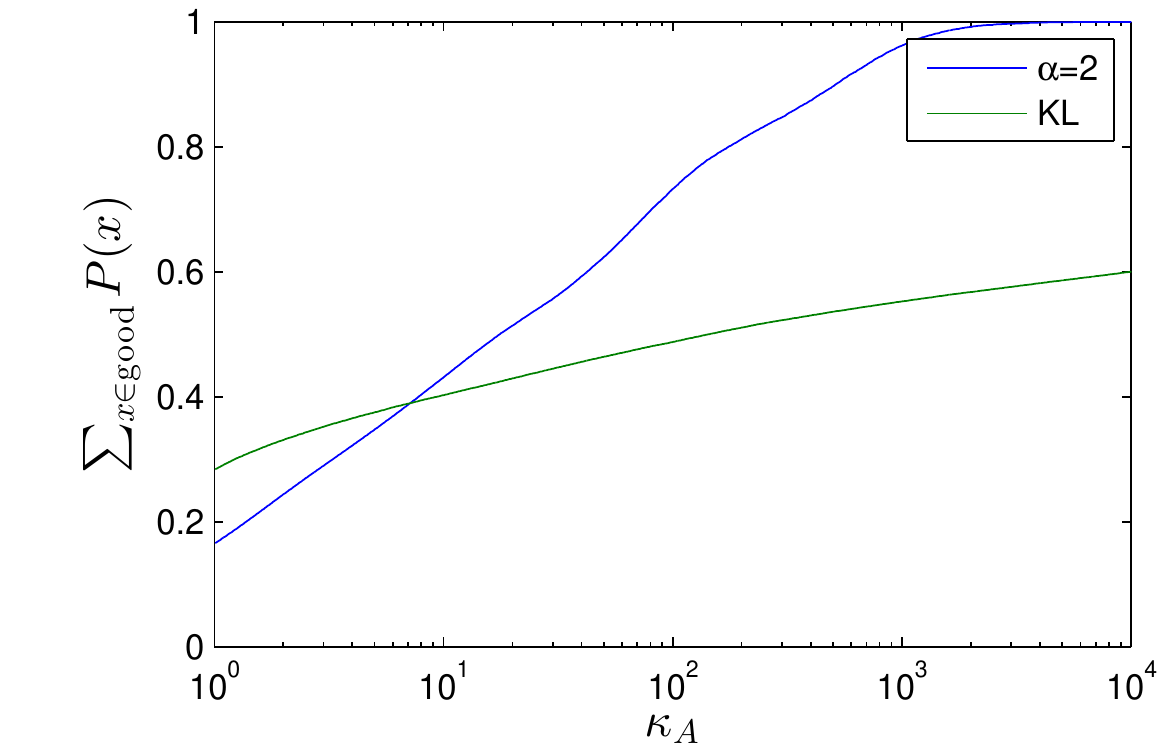}
\caption{The total probability mass that is well approximated, meaning $P(v,h) \le \kappa_A Q(v,h)$ as a function of $\kappa_A$ for an RBM with $n_v=6$ and $n_h=8$.  The left plot corresponds to $\lambda=0.1$ and the right plot $\lambda=0.01$.\label{fig:lambda}}
\end{figure}

\subsection{Scaling with $\lambda$}
In the main text we examined the scaling of $D_2$ for random RBMs with edges whose weights are drawn from $\mathcal{N}(0,1)$ for different sized graphical models.  Such Boltzmann machines are not expected to be typical of those that emerge from training in the presence of regularization, which is expected to lead to substantial weight decay for large models.  We examine the accuracy of rejection sampling, as a function of $\kappa_A$, for both mean--field approximations and $Q_{\alpha=2}$ in~\fig{lambda}. Specifically, consider an RBM with $n_v=6$ and $n_h=8$ whose weights were found by exact training and training vectors drawn from the synthetic set described in the main text with $\mathcal{N}=0.1$.  We also take  $Z_{\rm MF}=Z$ in order to simplify the comparison.

The data in~\fig{lambda} shows that for $\lambda=0.1$, the mean values of the probability such that $P(v,h)\ge \kappa Q(v,h)$ are graphically indistinguishable for both the mean--field approximation and $Q_{\alpha=2}$.  This is to be expected since $\lambda=0.1$ corresponds to relatively strong regularization and it is perhaps not surprising that if the Gibbs distribution is nearly a product distribution that optimizing either ${\rm KL}(Q|| P)$ or $D_2(P||Q)$ should lead to comparable results.  The sum of the bad probabilities was found to fall off, for large $\kappa_A$, as $\kappa_A^{-3}$ for $Q_{\alpha=2}$ and $\lambda=0.01$.  In contrast the asymptotic scaling for $Q_{\rm MF}$ for $\lambda=0.01$ is unclear for this regularization constant as the values of $\lambda$ considered are insufficient to see the asymptotic behavior of the curve.  

There are substantial differences between the fraction of configurations that are correctly handled for large $\kappa_A$ for $\lambda=0.01$.  As expected by Corollary 1, $Q_{\alpha=2}$ outperforms $Q_{\rm MF}$ in the asymptotic limit.  It perhaps is unsurprising that $Q_{\rm MF}$ can outperform $Q_{\alpha=2}$ for modestly small $\kappa_A$ because choosing the distribution to minimize the average value of $P(v,h)/Q(v,h)$ may not minimize the tail probabilities as effectively as one may like owing to the looseness of the Markov inequality.  In practice, this suggests that minimizing for even more general $\alpha$ divergences may be useful for trading off asymptotic versus short--time performance of the sampling algorithm.

We examine the same problem but for fixed $\lambda$ and variable $n_h$ in~\fig{lambda2}.  We note that for these networks that the value of $\kappa_A$ needed to obtain $99\%$ of the probability mass scales roughly as $0.27 n_h$ for $\lambda=0.1$.  This scaling should be taken with a grain of salt as we do not have enough data to meaningfully extrapolate to large system sizes.  However, the salient feature is that there is no sign of exponential divergence despite the number of edges in the graphical model tripling.  

For $\lambda=0.01$ we notice no such trend: the sum of the good probabilities at $n_h=8$ is on average less than that observed for $n_h=12$.  This is likely because the model contains $48$ edges for $n_h=8$ and if $\lambda=0.01$ the effect of the regularization term on $O_{\rm ML}$ is potentially less significant than it would be if the same distribution of weights were taken at $n_h=12$ where $72$ edges are present.  As such it is not helpful to consider how the cost of refining the mean--field approximation into the Gibbs state scales with the size of the graphical model.  More important features such as the sparsity of the graph and the presence or absence of frustration are expected to better determine the viability of these variational approximations~\cite{Jor99}.

Another interesting feature is that at $\kappa=1$ the mean values of the sum of the good probabilities in the distribution that results from rejection sampling for $n_h=8$ and $n_h=12$ are within $(1.5\pm 0.2)\%$ of each other if $\lambda=0.1$.  This is despite the fact that the Hilbert space for the $n_h=12$ case is $16$ times larger and the number of edges in the model is $50\%$ larger.  This shows that in the presence of regularization the error in variational approximations to the Gibbs distribution need not strongly depend on the size of the graph.  Similar results have been noticed for mean--field state preparations in~\cite{wiebe_quantum_2014}.  The results for $\lambda=0.01$ is much more significant with an average difference of $(11.6\pm0.2)\%$.  We suspect these differences occur because the graphs considered are not yet large enough for weak regularization to push the Gibbs state towards an approximately unimodal distribution.

\begin{figure}[t!]
\includegraphics[width=0.33\columnwidth]{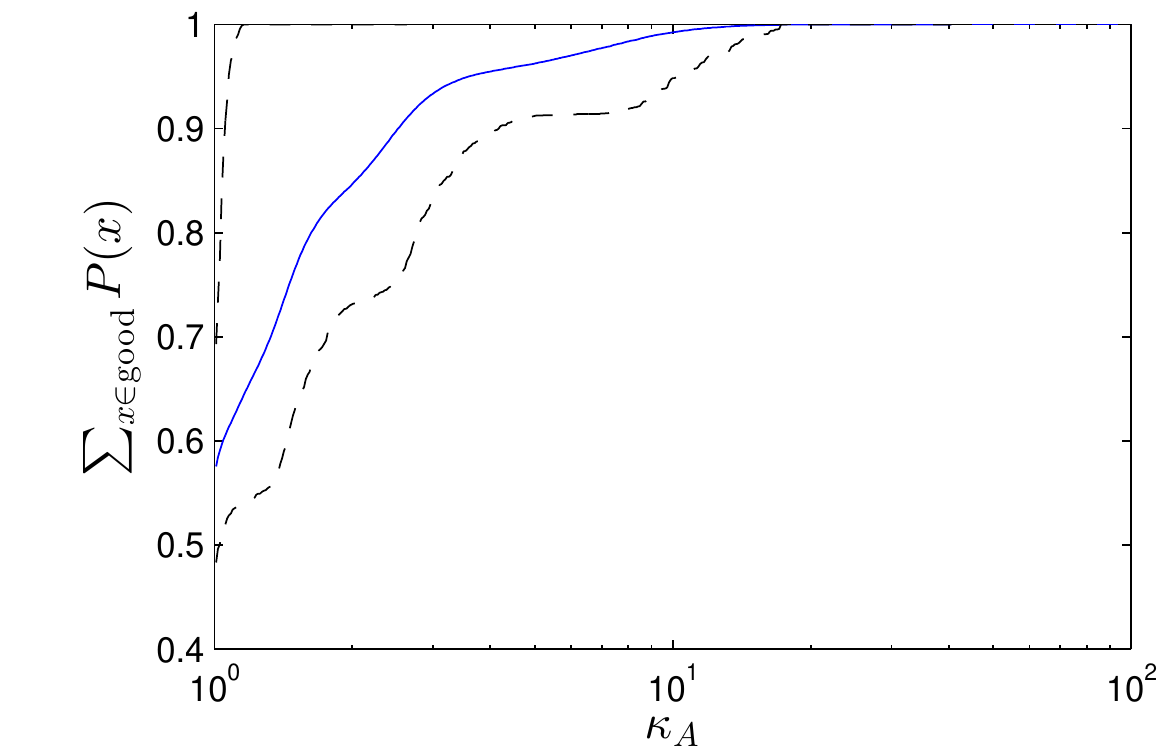}
\includegraphics[width=0.33\columnwidth]{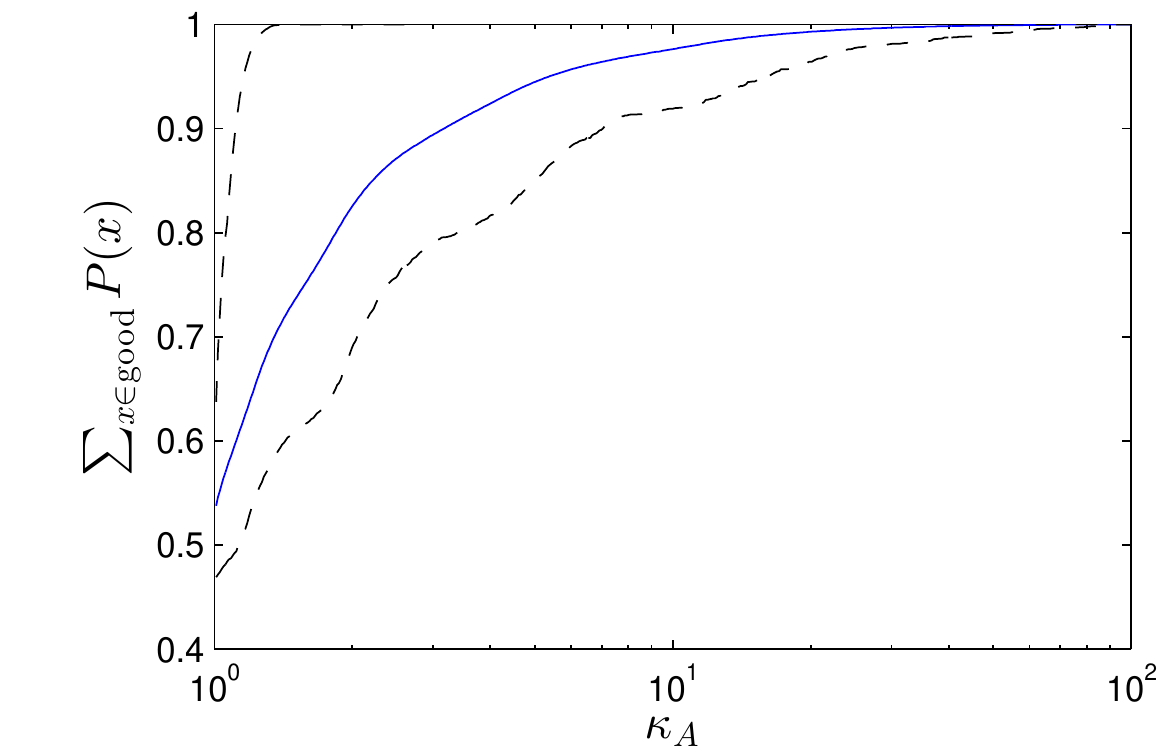}
\includegraphics[width=0.33\columnwidth]{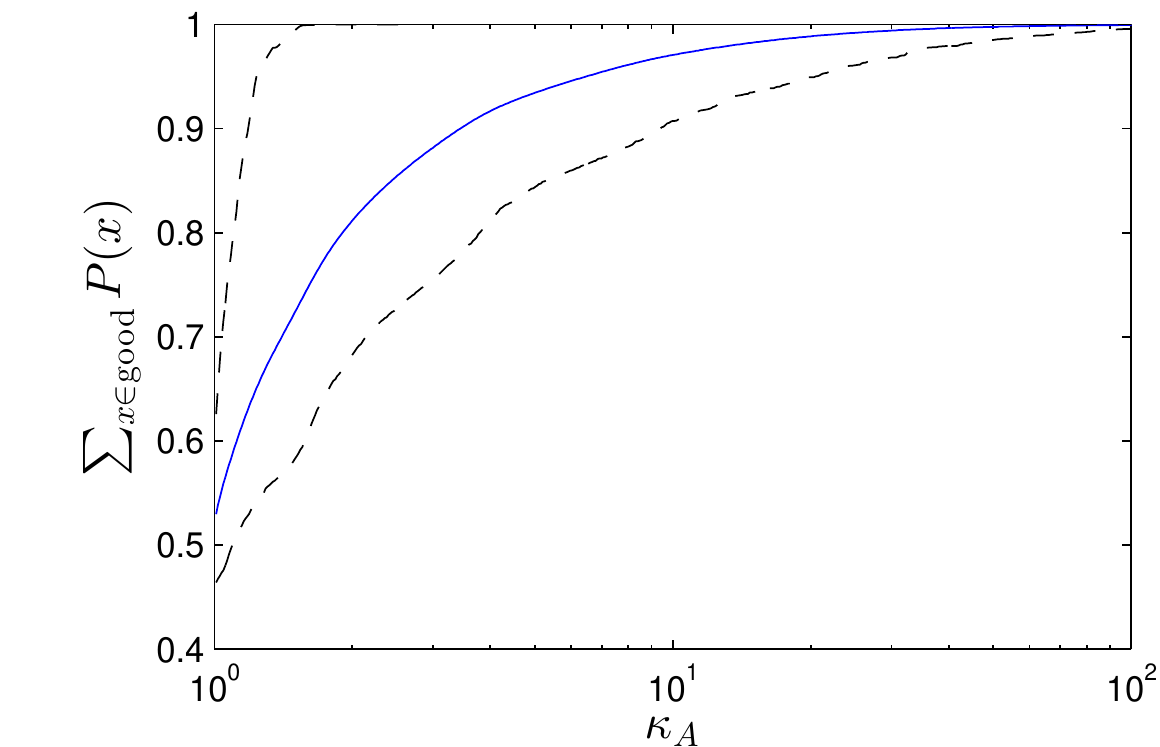}
\includegraphics[width=0.33\columnwidth]{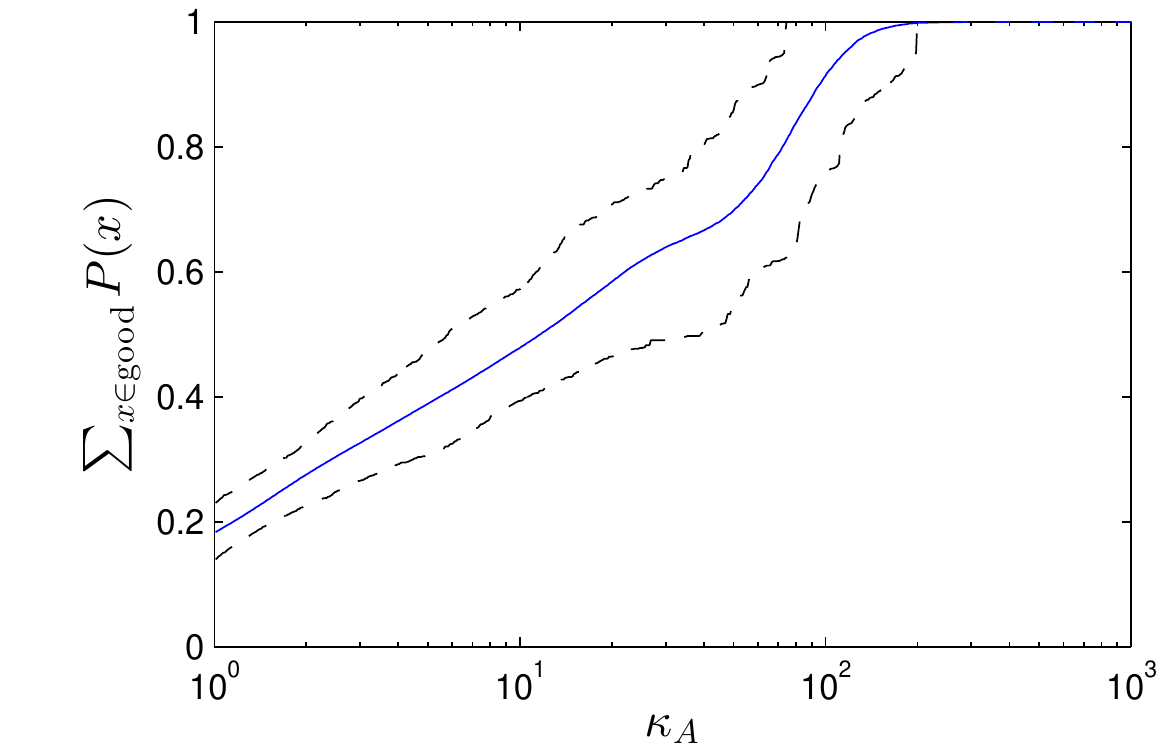}
\includegraphics[width=0.33\columnwidth]{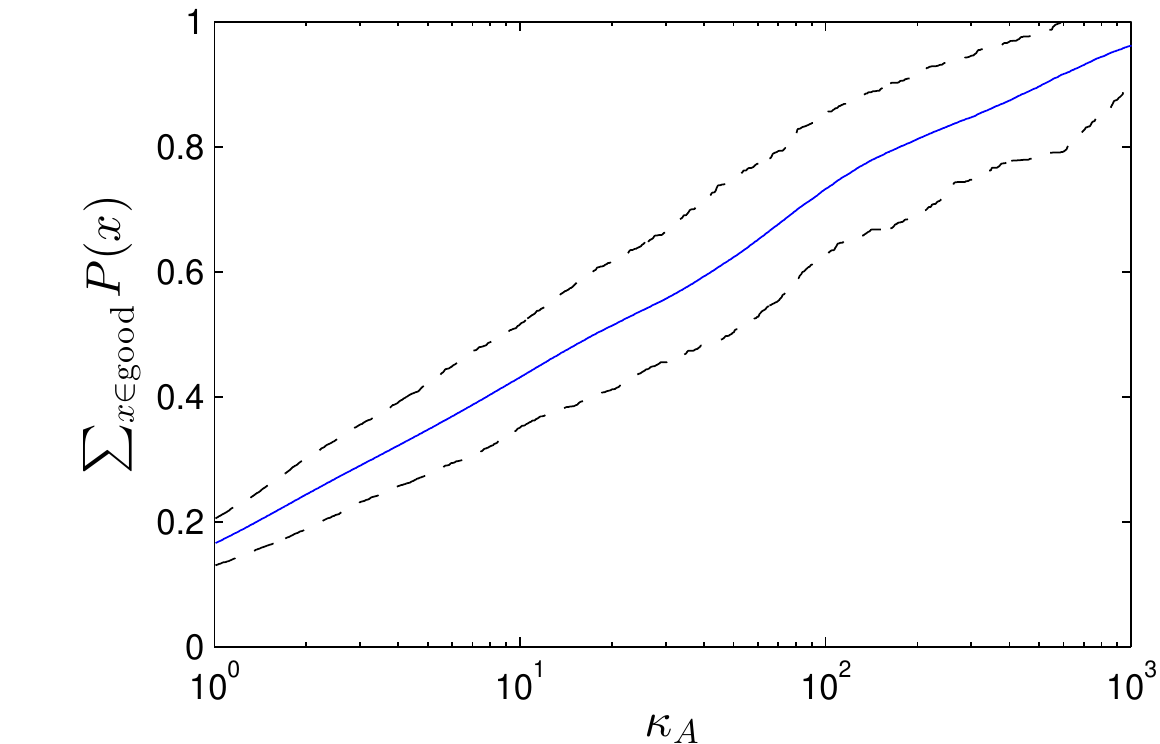}
\includegraphics[width=0.33\columnwidth]{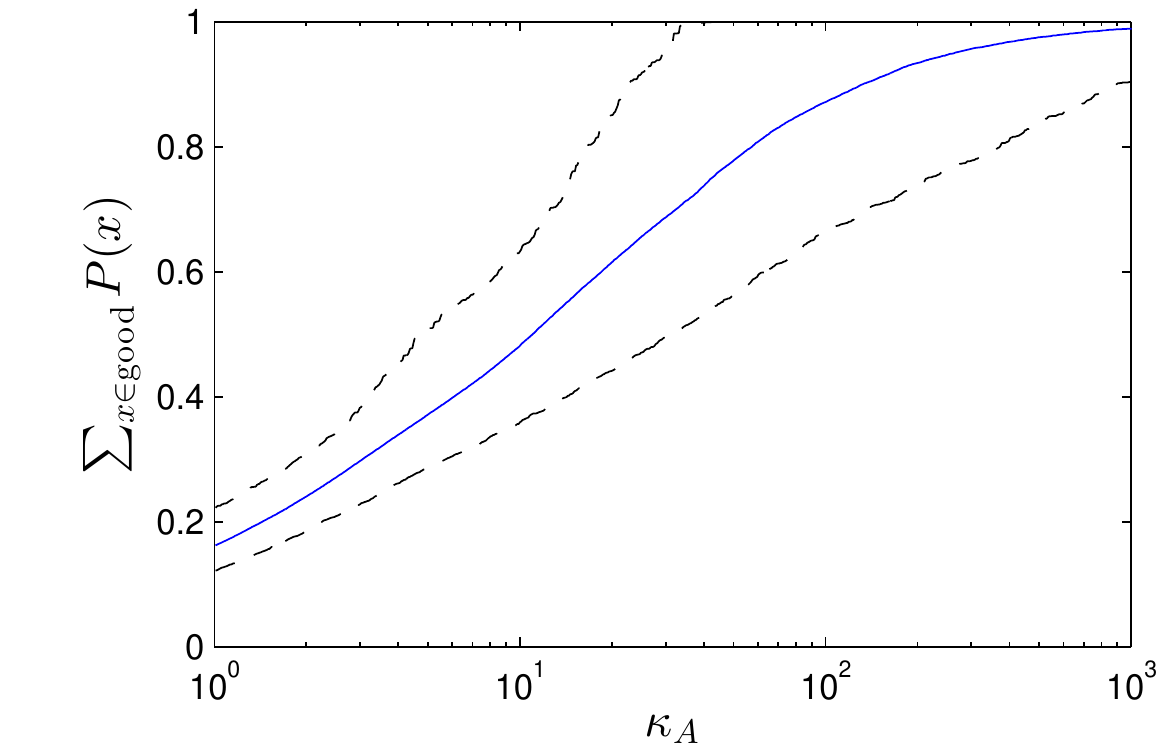}
\caption{The mean value and  95\% confidence interval for the total sum of the probability for the configurations where $P(v,h)\ge \kappa Q_{\alpha=2}(v,h)$ for $\lambda=0.1$ (top) and $\lambda=0.01$ (bottom) for trained RBMs with $n_h=4$ (left) $n_h=8$ (center) $n_h=12$ (right).\label{fig:lambda2}}
\end{figure}

\subsection{Scaling with $\kappa_A$}
The quantity $\kappa_A$ is perhaps the most important factor for determining the viability of our method relative to contrastive divergence because it dictates the success probability of the
rejection step.  Here we will provide numerical evidence in small samples that illustrates that small values of $\kappa_A$ provide comparable, or greater, accuracy than contrastive divergence training.  For all these results we again use an equal mixture of $Q_{\rm MF}$ and the uniform distribution as our instrumental distribution.  We anticipate that the use of $Q_{\alpha=2}$ will tend to result in better results for a fixed value of $\kappa_A$.

\fig{kappa} shows the difference in the value of the objective functions obtained relative to those found by exactly following the ML--objective function for an RBM with $n_v=6$ and $n_h=4$ and $\lambda=0.05$.  We observe for that data that $\kappa_A=75$ suffices to provide a mean discrepancy that is comparable to contrastive divergence.  It is important to note though that a substantial fraction of the results for even $\kappa_A=50$ are dramatically better than the results for contrastive divergence; however, the worst cases are nearly twice as bad.  The results for $\kappa_A \le 100$ show evidence of saturating and by computing the values that they saturate at we see that the data agrees well with an $e^{-0.05\kappa}$ scaling and has relatively poor agreement with powerlaw or linear scalings.  This captures the observed fact that the quality of the gradients rapidly improves as $\kappa_A$ is increased.  In particular, for $\kappa_A=200$, the discrepancies between the worst case scalings and the best case scalings of the data collapse and nearly all of the $1000$ examples tested were found to perform better than contrastive divergence.

These results do reveal an interesting feature of contrastive divergence, namely its consistency.  The best and worst case performances seem to be tightly clustered relative to those observed for our sampling algorithm.  It also seems to perform better for the first few epochs of optimization than IRS does, even in the limit of large $\kappa_A$.  This illustrates that contrastive divergence is likely to remain an algorithm of choice for serial (rather than parallel) training environments where variational approximations to the Gibbs states are expected to abjectly fail.

We examine a similar case with much weaker regularization in~\fig{kappa2}.  There we note that much larger values of $\kappa_A$ are needed to obtain good approximations to the gradient.  In particular, we observe that $\kappa_A=400$ is the first example where the best case performance beats that of contrastive divergence.  The data set is otherwise qualitatively similar to that considered in~\fig{kappa} except no evidence of training plateauing is observed within the number of epochs considered (with the exception of some of the data for $\kappa_A=200$).

\begin{figure}[t!]
\centering
\includegraphics[width=0.495\columnwidth]{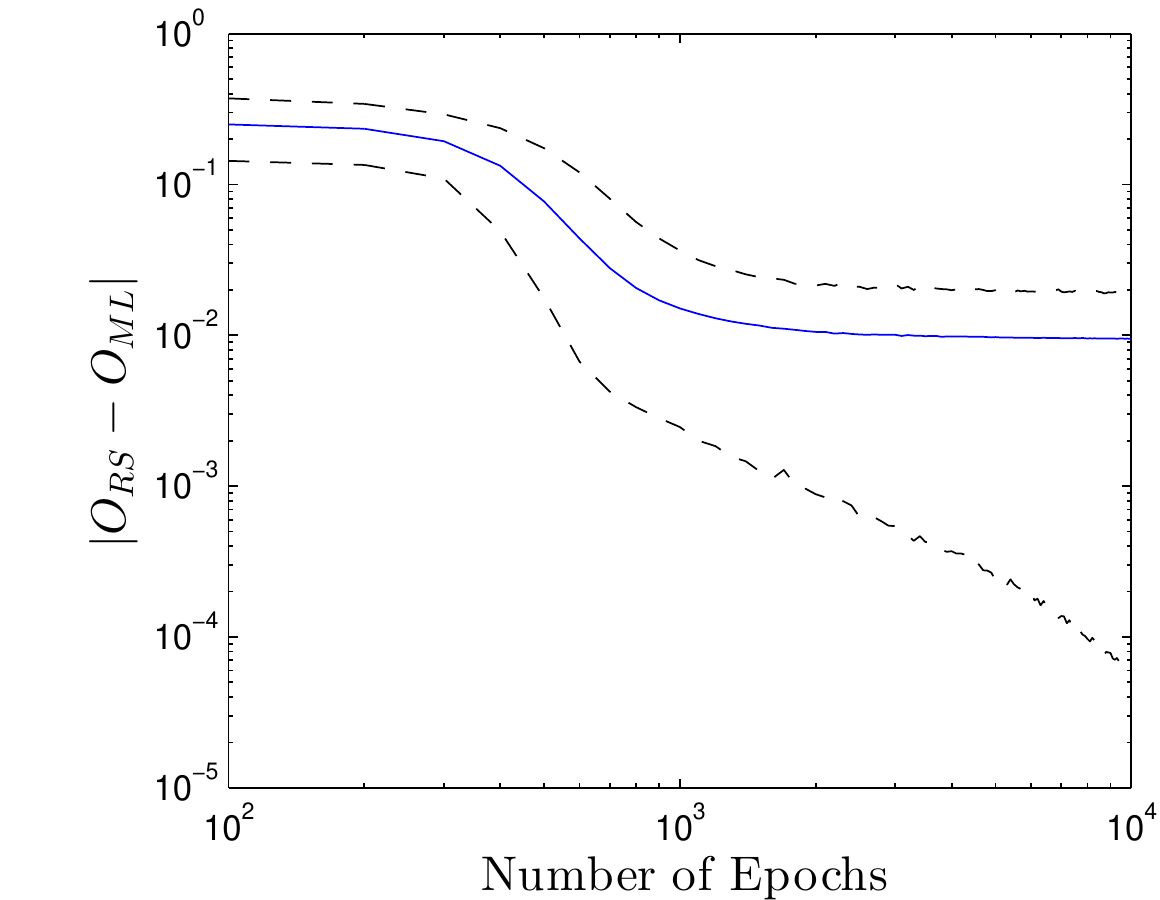}
\includegraphics[width=0.495\columnwidth]{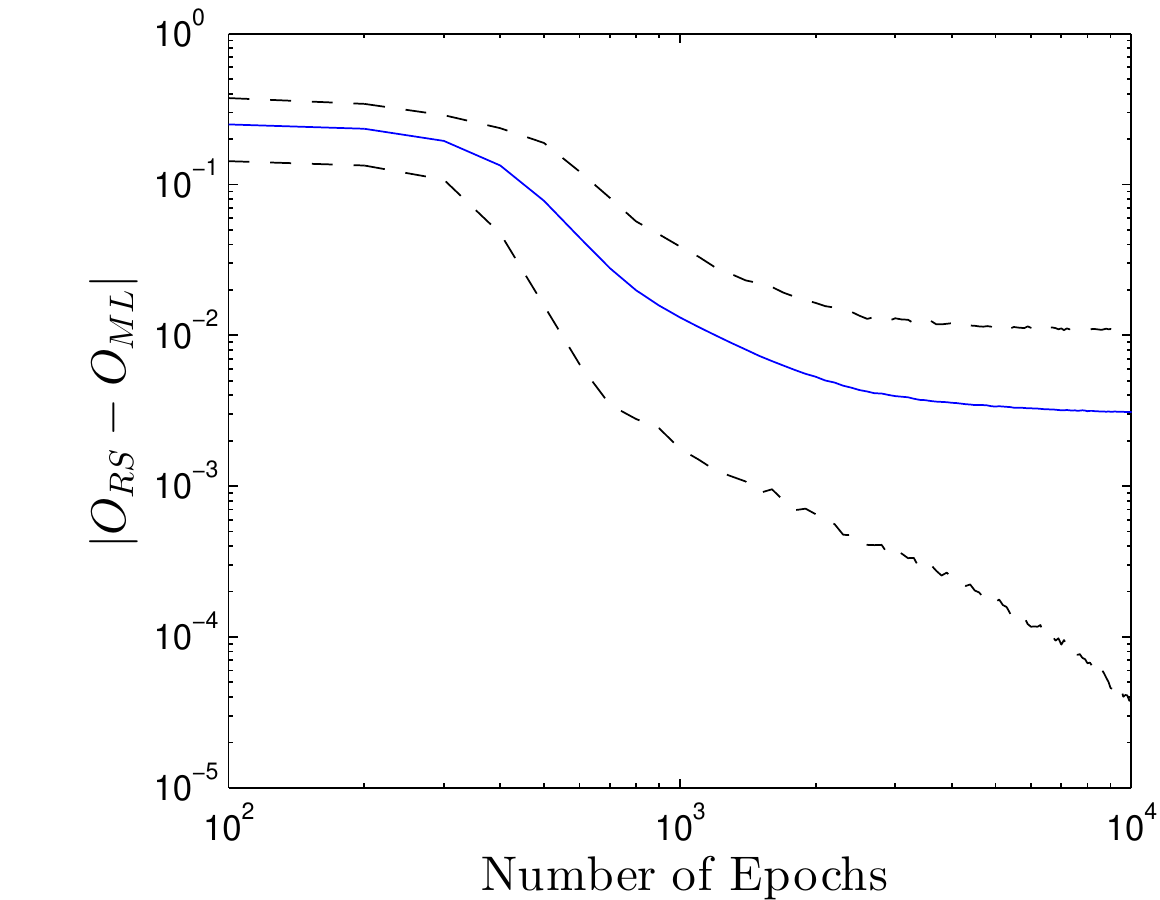}
\includegraphics[width=0.495\columnwidth]{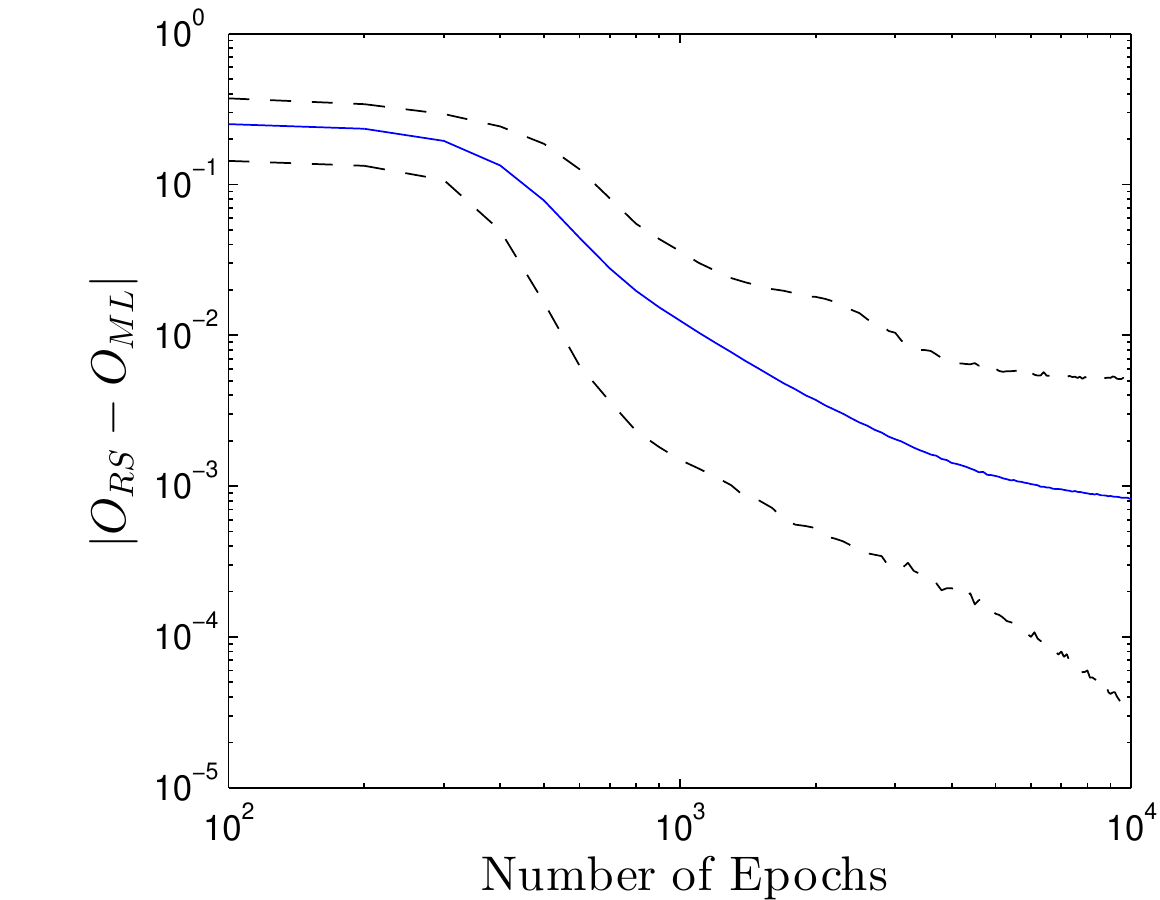}
\includegraphics[width=0.495\columnwidth]{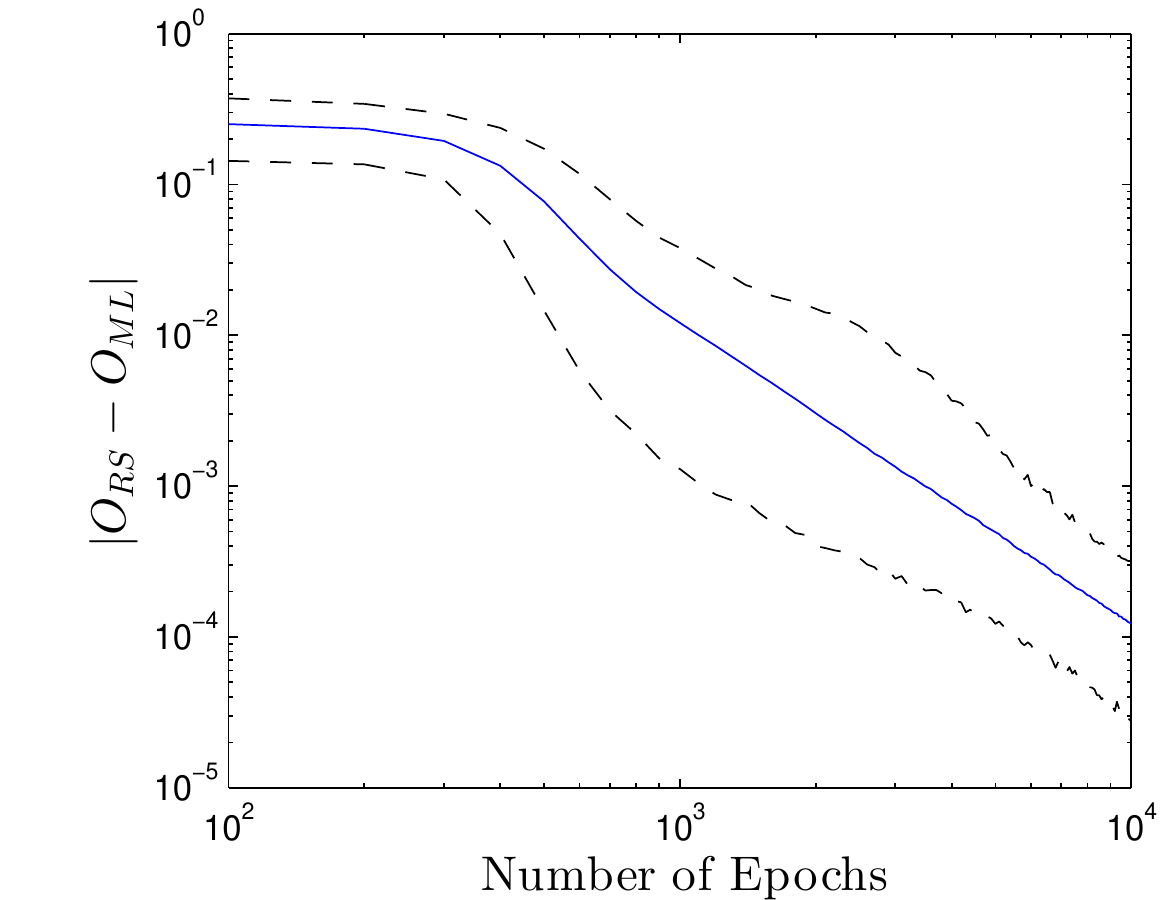}
\includegraphics[width=0.495\columnwidth]{CD6x4.pdf}
\caption{Mean and $95\%$ confidence interval for the discrepancy between the values of the ML training objective found using IRS training and CD-1 training for RBMs with $n_v=6$ and $n_h=4$, $\lambda=0.05$ and $\kappa_A = 50$ (Top left), $\kappa_A=75$ (Top right), $\kappa_A=100$ (Middle left), $\kappa_A=200$ (Middle right) and contrastive divergence training (Bottom).\label{fig:kappa}}
\end{figure}

\begin{figure}[t!]
\centering
\includegraphics[width=0.495\columnwidth]{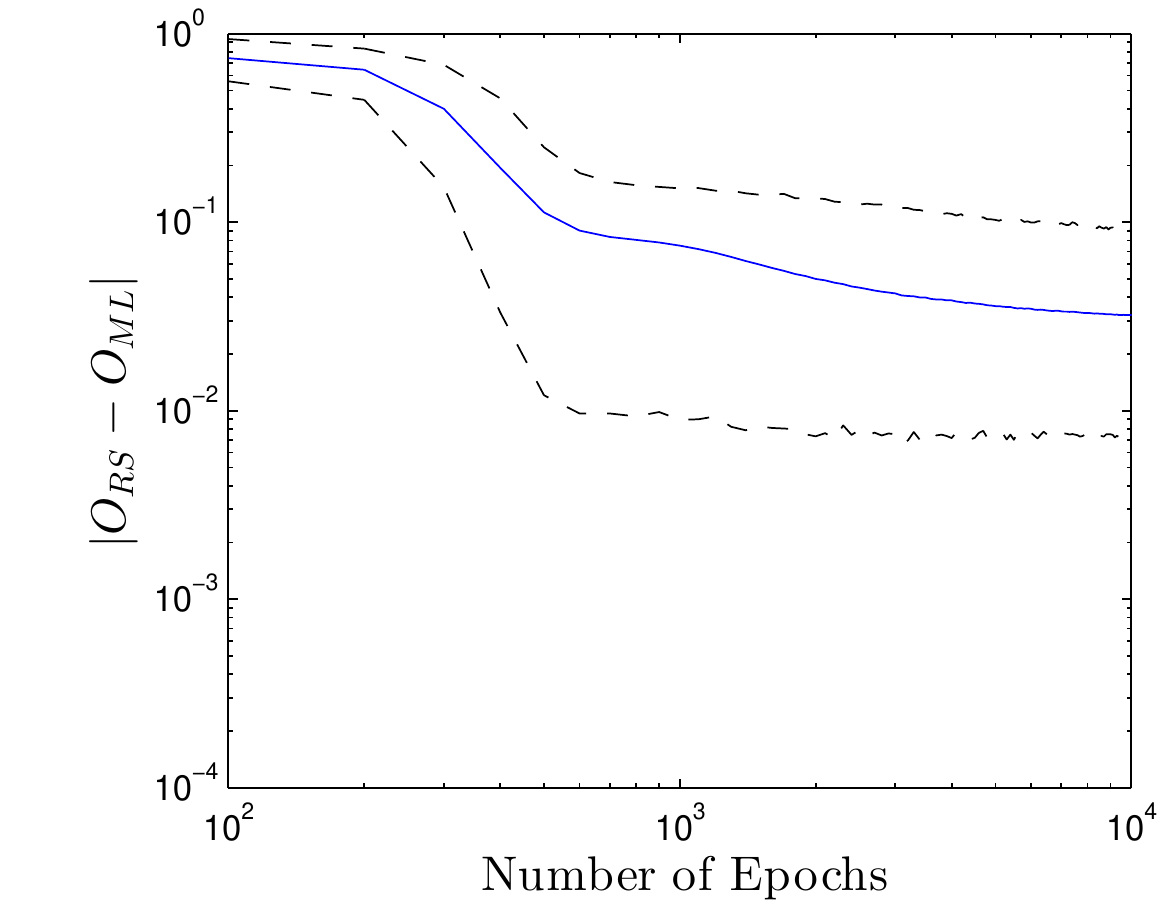}
\includegraphics[width=0.495\columnwidth]{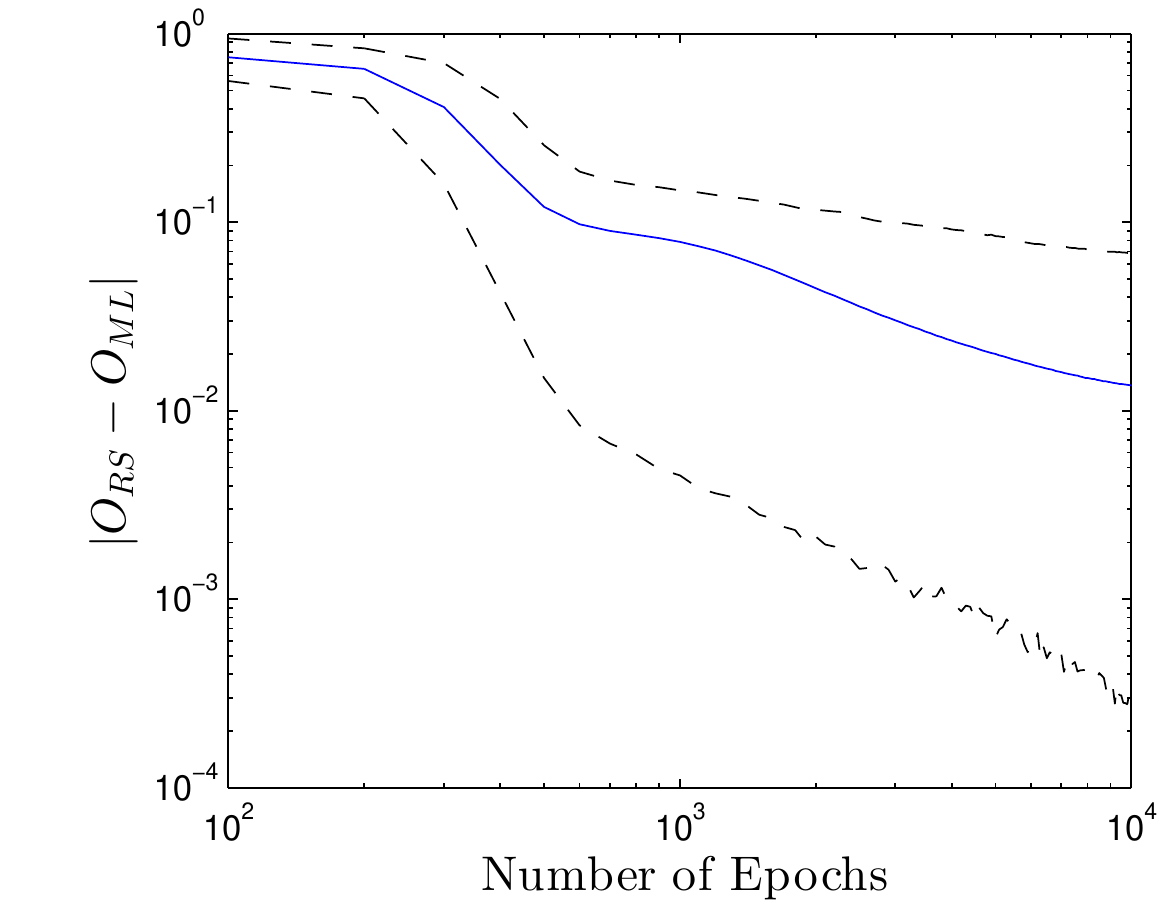}
\includegraphics[width=0.495\columnwidth]{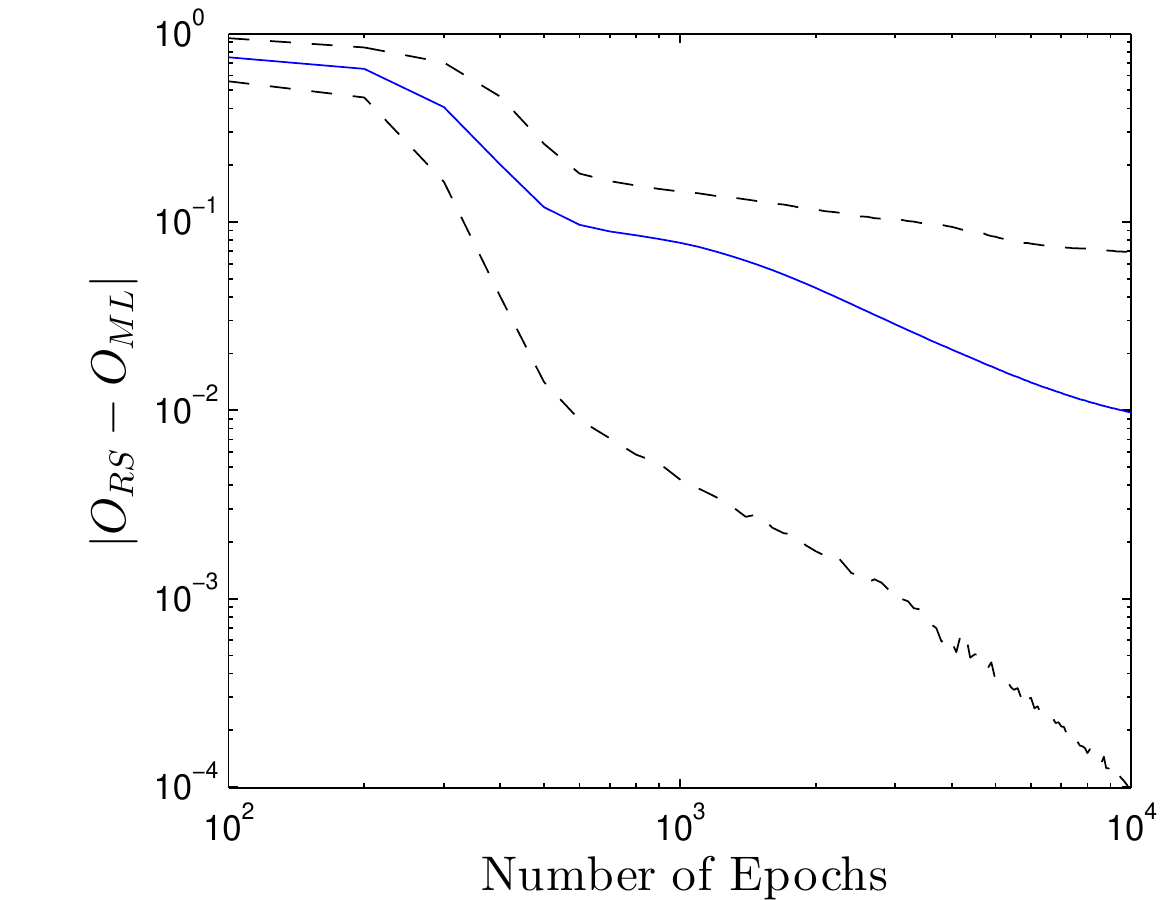}
\includegraphics[width=0.495\columnwidth]{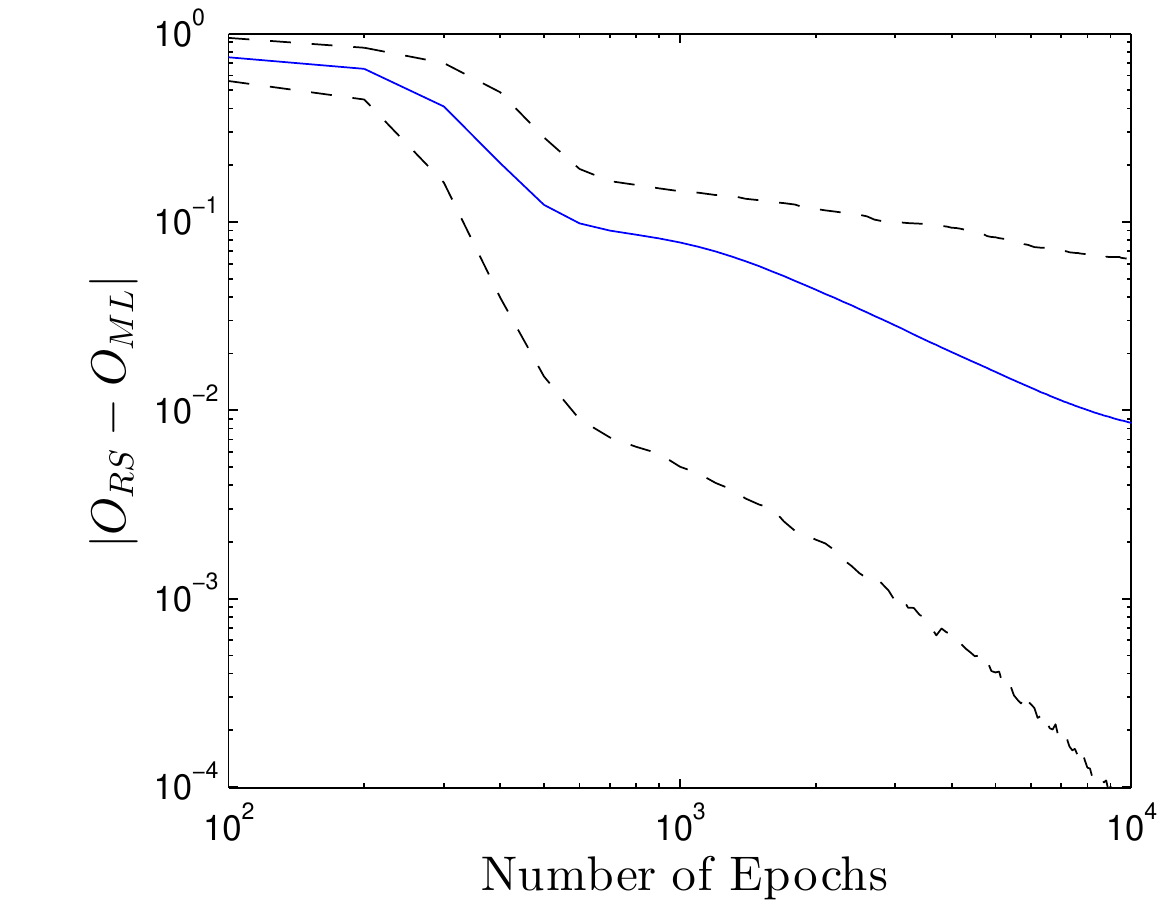}
\includegraphics[width=0.495\columnwidth]{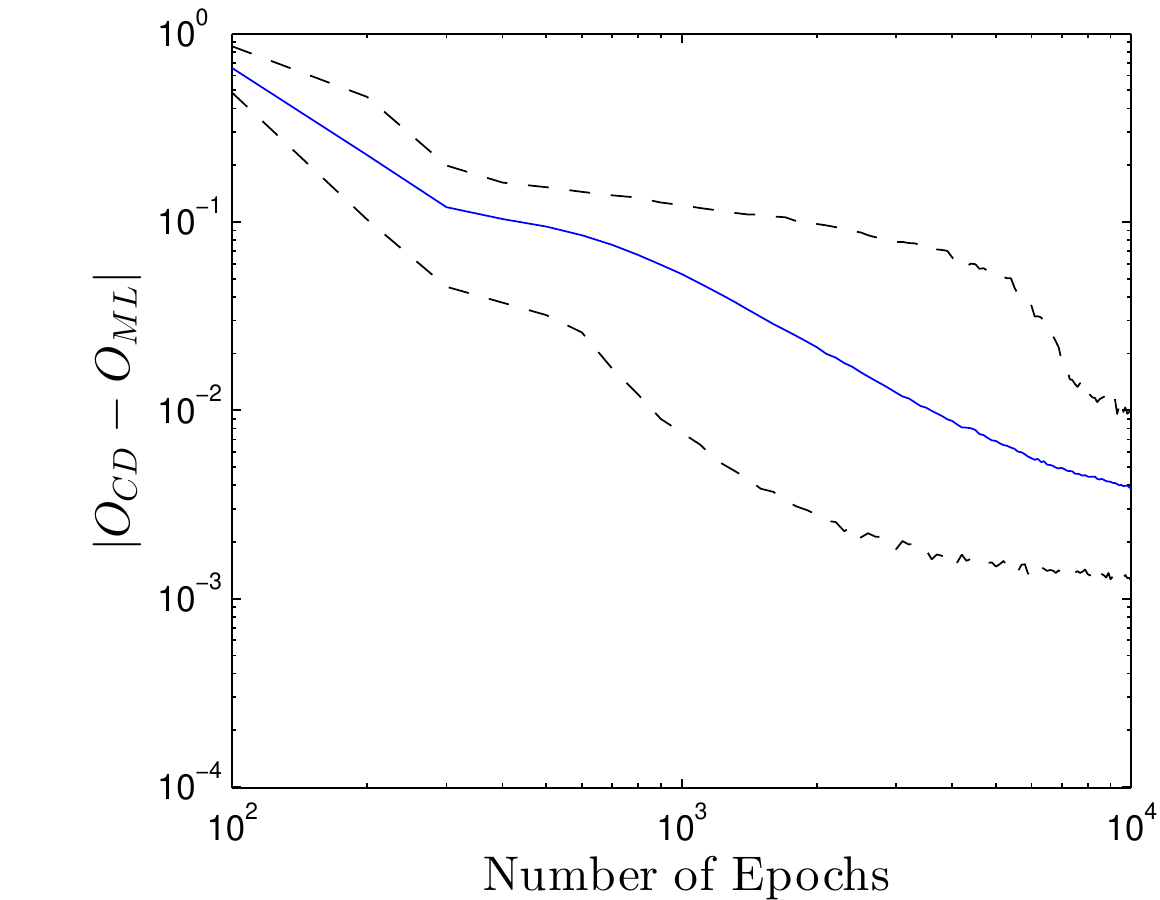}
\caption{Mean and $95\%$ confidence interval for the discrepancy between the values of the ML training objective found using IRS training and CD-1 training for RBMs with $n_v=6$ and $n_h=4$, $\lambda=0.01$ and $\kappa_A = 200$ (Top left), $\kappa_A=400$ (Top right), $\kappa_A=800$ (Middle left), $\kappa_A=1600$ (Middle right) and contrastive divergence training (Bottom).\label{fig:kappa2}}
\end{figure}

\section{Hedging strategies}
In the main text, we showed that choosing our product distribution $Q$ to minimize the $D_2(P||Q)$ rather than ${\rm KL}(Q||P)$.  However, as we noted in~\fig{lambda}, the mean--field approximation may actually yield \emph{a better approximation} to the Gibbs state than $Q_{\alpha=2}$ does if $\kappa_A$ is small.  The reason for this is that $D_2$ attempts to minimize the average ratio between the two, but in practice it may not necessarily model the high probability regions of the distribution accurately.  In contrast, the mean--field distribution aims to find the distribution that is closest to the Gibbs distribution but requires a large value of $\kappa_A$ to accurately model the tails of the probability distribution.  This begs the question of whether choosing a different instrumental distribution that combines the best features of both distributions can be used.

\begin{figure}[t!]
\centering
\includegraphics[width=0.6\columnwidth]{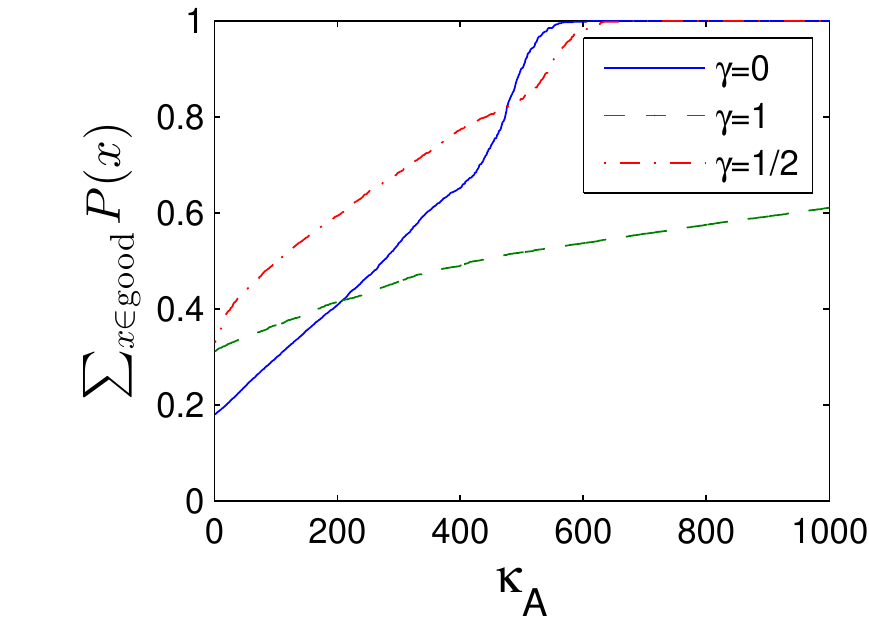}
\caption{The sum of the probability mass of the configurations that is correctly handled by the rejection for hedging strategies using different values of $\gamma$ for an RBM with $n_v=6$ and $n_h=8$ with $\lambda=0.01$.  The weights and biases of the RBM were those at the local optima of $O_{\rm ML}$\label{fig:gamma}.}
\end{figure}

A strategy proposed in~\cite{wiebe_quantum_2014} is to choose the instrumental distribution to be a combination of the mean--field distribution and one that captures the tail probability more effectively.  In~\cite{wiebe_quantum_2014} they mix uniform distribution with the mean--field distribution.  While this works well for small systems, it is not expected to work well for high--dimensional systems because the prediction shrinks exponentially with the dimension of the Hilbert space that the probability distribution is supported over.  Instead, a more sensible approach is to combine the mean--field and $Q_{\alpha=2}$ in the following way
\begin{equation}
Q(v,h) = \gamma Q_{\rm MF}(v,h)+(1-\gamma)Q_{\alpha=2}(v,h),
\end{equation}
for $\gamma\in [0,1]$.


We examine this strategy in~\fig{gamma} where we observe that for small $\kappa_A\approx 1$ the quality of the approximation yielded at $\gamma=1/2$ is comparable to that at $\gamma=1$.  This is surprising because averaging the mean--field approximation with one that is known to give a worse approximation may be expected to result in an inferior approximation.  For larger values of $\kappa_A$, the $\gamma=1/2$ hedged approximation outperforms either $Q_{\rm MF}$ or $Q_{\alpha=2}$ individually.  In fact, at $\kappa_A \approx 211$ the distribution with $\gamma=1/2$ outperforms both of the distributions by roughly $45\%$.  This not only suggests that hedging can lead to substantial improvements but also suggests that choosing different values of $\alpha$ to interpolate between the performance of $\alpha=2$ and $\alpha=0$ (mean--field) may also improve the performance of IRS for small kalues of $\kappa_A$.

\section{Detailed algorithm for computing gradients}
We provide  a more detailed algorithm below for computing the gradient of the ML objective function using IRS.  The algorithm utilizes subroutines $Q$ and  $\mathcal{Q}$ that provide an estimate of the probability of a given configuration of hidden and visible units.  These functions will often correspond to a tractable approximation to the Gibbs
distribution such as the mean--field approximation or a product distribution that minimizes the $\alpha=2$ divergence with the Gibbs state.  We also assume that a sampling procedure is known for these distributions.

\begin{algorithm}[h!]
\rule{\linewidth}{1pt}
\begin{algorithmic}
\Require Initial model weights $w$, visible biases $b$, hidden biases $d$,  $\kappa_A$, a set of training vectors $x_{\rm train}$, a regularization term $\lambda$, a learning rate $r$ and the functions $Q(v,h)$, $\mathcal{Q}(h;v)$, $Z_Q$, $Z_{Q(h;v)}$.
\Ensure $\texttt{gradMLw},\texttt{gradMLb},\texttt{gradMLd}$.
\vskip0.2em
\hrule
\vskip0.2em
{
\small
\For{$i=1:N_{\rm train}$}
\State ${\texttt{success}}\gets 0$
\While{$\texttt{success}=0$}\Comment{Draw samples from approximate model distribution.}
\State Draw sample $(v,h)$ from $Q(v,h)$.  
\State $E_{s} \gets E(v,h)$ 
\State Set ${\texttt{success}}$ to $1$ with probability $\min(1, e^{-Es}/(Z_Q \kappa_A Q(v,h)))$.
\EndWhile
\State $\texttt{modelV}[i] \gets v$.
\State $\texttt{modelH}[i] \gets h$.
\State ${\texttt{success}}\gets 0$
\State $v\gets x_{\rm train}[i]$.
\While{$\texttt{success}=0$}\Comment{Draw samples from approximate data distribution.}
\State Draw sample $h$ from $\mathcal{Q}(h;v)$.
\State $E_{s} \gets E(v,h)$.
\State Set $\texttt{success}$ to $1$ with probability $\min(1, e^{-Es}/(Z_{Q(v,h)} \kappa_A {\mathcal{Q}}(v,h)))$.
\EndWhile
\State $\texttt{dataV}[i] \gets v$.
\State $\texttt{dataH}[i] \gets h$.
\EndFor
\For{each visible unit $i$ and hidden unit $j$}
\State $\texttt{gradMLw}[i,j] \gets r\left(\frac{1}{N_{\rm train}}\sum_{k=1}^{N_{\rm train}}\left(\texttt{dataV}[k,i]\texttt{dataH}[k,j]-\texttt{modelV}[k,i]\texttt{modelH}[k,j]\right)-\lambda w_{i,j}\right)$.
\State $\texttt{gradMLb}[i] \gets r\left(\frac{1}{N_{\rm train}}\sum_{k=1}^{N_{\rm train}}\left(\texttt{dataV}[k,i]-\texttt{modelV}[k,i]\right)\right)$.
\State $\texttt{gradMLd}[j] \gets r\left(\frac{1}{N_{\rm train}}\sum_{k=1}^{N_{\rm train}}\left(\texttt{dataH}[k,j]-\texttt{modelH}[k,j]\right)\right)$.
\EndFor
}
\end{algorithmic}
\rule{\linewidth}{1pt}
\caption{\label{alg:graderror}Algorithm for estimating $\nabla O_{\rm ML}$.}
\end{algorithm}



\end{document}